\documentclass[11pt]{article}

\usepackage{latexsym,mathrsfs}
\usepackage{amsmath,amssymb} 
\usepackage{amsthm,enumerate,verbatim}
\usepackage{amsfonts}
\usepackage{mathabx}
\usepackage{graphicx}
\usepackage{url}
\usepackage{graphicx}
\usepackage{subcaption}
\usepackage{mathtools}
\usepackage[toc,page]{appendix}
\usepackage{multirow}
\usepackage{xcolor}
\usepackage{hyperref}
\usepackage{algorithm}
\usepackage{algorithmic}

\definecolor{brightpink}{rgb}{1.0, 0.0, 0.5}

\newcommand{\ngi}[1]{{{\color{black} #1}}}
\definecolor{darkgreen}{rgb}{0,.4,0}
\def\ji#1{{\noindent\color{black}{#1}}}
\def\jiter#1{{\noindent\color{darkgreen}{#1}}}
\newcommand{\revise}[1]{{{\color{red} #1}}}

\newtheorem{definition}{Definition}
\newtheorem*{definition*}{Definition}

\newtheorem{assumption}{Assumption}
\newtheorem{remark}{Remark}
\newtheorem{proposition}{Proposition}

\DeclareMathOperator{\logdet}{logdet} 
\DeclareMathOperator{\tr}{trace} 
\DeclareMathOperator{\diag}{Diag} 

\input alphabet.tex

\title{\LARGE \bf
Multiplicative Updates for NMF with $\beta$-Divergences 
under Disjoint Equality Constraints 
}

\author{Valentin Leplat$^*$ \and 
Nicolas Gillis\thanks{Department of Mathematics and Operational Research,
Facult\'e Polytechnique, Universit\'e de Mons,
Rue de Houdain 9, 7000 Mons, Belgium. Authors acknowledge the support by the Fonds de la Recherche Scientifique - FNRS and the Fonds Wetenschappelijk Onderzoek - Vlanderen (FWO) under EOS Project no O005318F-RG47, and by the European Research Council (ERC Starting Grant no 679515, ERC Consolidator Grant no 681839). 
E-mails: \{valentin.leplat, nicolas.gillis\}@umons.ac.be.
} \and  J\'er\^ome Idier\thanks{Laboratoire des Sciences du Num\'erique de Nantes (LS2N, CNRS UMR 6004), Ecole Centrale de Nantes, 44321 Nantes, France.
E-mail: jerome.idier@ls2n.fr}
}

\begin{document}

\maketitle

\begin{abstract} 
Nonnegative matrix factorization (NMF) is the problem of approximating an input nonnegative matrix, $V$, as the product of two smaller nonnegative matrices, $W$ and $H$. In this paper, we introduce a general framework to design multiplicative updates (MU) for NMF based on $\beta$-divergences ($\beta$-NMF) with disjoint equality constraints, and with penalty terms in the objective function. By disjoint, we mean that each variable appears in at most one equality constraint. Our MU satisfy the set of constraints after each update of the variables during the optimization process, while guaranteeing that the objective function decreases monotonically. We showcase this framework on three NMF models, and show that it competes favorably the state of the art: (1)~$\beta$-NMF with sum-to-one constraints on the columns of $H$, (2)~minimum-volume $\beta$-NMF with sum-to-one constraints on the columns of $W$, and (3)~sparse $\beta$-NMF with $\ell_2$-norm constraints on the columns of $W$. 
\end{abstract}
\textbf{Keywords}: 
nonnegative matrix factorization (NMF), 
$\beta$-divergences, 
disjoint constraints, 
simplex-structured NMF, 
minimum-volume NMF, 
sparsity 


\section{Introduction}
\label{sec_introduction}

Given a non-negative matrix $V \in \mathbb{R}_{+}^{F \times N}$ and a factorization rank $K \ll \min(F,N)$, nonnegative matrix factorization (NMF) aims to compute two non-negative matrices, $W$ with $K$ columns and $H$ with $K$ rows, such that $V\approx WH$~\cite{lee1999learning}. 
Over the last two decades, NMF has shown to be a powerful tool for the analysis of high-dimensional data. The main reason is that NMF automatically extracts sparse and meaningful features from a set of nonnegative data vectors. NMF has been successfully used in many applications such as image processing, text mining, hyperspectral imaging, blind source separation, single-channel audio source separation, clustering and music analysis; see~\cite{gillis2014, cichocki2009nonnegative, chi2012tensors, neymeyr2018set, fu2019nonnegative, gillis2020bk} and the references therein. 

To compute $W$ and $H$, the most standard approach is to solve the following optimization problem 
\begin{equation}\label{eq:2}
 \underset{W\in \mathbb{R}^{F \times K},H\in \mathbb{R}^{K \times N}}{\min}
 D\left(V|WH\right)  \quad  \text{ such that } \quad H \geq 0 \text{ and } W \geq 0, 
\end{equation}
where $D\left(V|WH\right)=\sum_{f,n} d(V_{fn}|[ \,WH] \,_{fn})$ with 
$d(x|y)$ a measure of distance between two scalars, 
and 
$A \geq 0$ means that the matrix $A$ is component-wise nonnegative. 
In this paper, we focus on $\beta$-NMF for which the measure of fit is the discrete $\beta$-divergence denoted $d_{\beta}(x|y)$ and defined as 
\[
d_{\beta}\left(x|y\right) = 
  \begin{cases}
    \frac{1}{\beta \left(\beta-1\right)}   \left(x^{\beta}+\left(\beta-1\right)y^{\beta}-\beta xy^{\beta-1}\right) & \text{for} \quad \beta \in \mathbb{R} \setminus \left\{0,1\right\}, \\
    x\log\frac{x}{y}-x+y           
      &   \text{for}\quad \beta=1,  \\
    \frac{x}{y}-\log\frac{x}{y}-1  
      &  \text{for}\quad \beta=0. 
  \end{cases}
\] 
For $\beta=2$, $d_2(x|y)=\frac12(x-y)^2$, so $D\left(V|WH\right)$ is the halved standard squared Euclidean distance between $V$ and $WH$, that is, 
the halved squared Frobenius norm $\frac{1}{2}\|V-WH\|_F^2$ . 
For $\beta=1$ and $\beta=0$, 
the $\beta$-divergence corresponds to the Kullback-Leibler (KL) divergence and the Itakura-Saito (IS) divergence, respectively. 
The error measure should be chosen accordingly with the distribution of the noise assumed on the data. 
The Frobenius norm assumes i.i.d.\@ Gaussian noise, KL divergence assumes a Poisson distribution, 
and the IS divergence assumes multiplicative gamma noise; see for example~\cite{fevotte2009nonnegative, dikmen2015learning, hong2020generalized} and the references therein. 
In the NMF literature, $\beta$-divergences are the most widely used objective functions. 


Most NMF algorithms developed to tackle~\eqref{eq:2} are based on iterative schemes that alternatively updates the factors $W$ and $H$. 
At each iteration, the minimization over one factor, $W$ or $H$, is performed with various optimization methods. 
For $\beta$-divergences, the most popular approach is to use multiplicative updates (MU) which were introduced for NMF in the seminal papers of Lee and Seung~\cite{lee1999learning, algoNMFlee}. 
In all applications we are aware of, $\beta$ is always chosen smaller than two. The reason is that, for $\beta > 2$,  $\beta$-divergences become more and more sensitive to outliers. Already for $\beta = 2$, it is well-known that the squared Frobenius norm is sensitive to outliers. However, the case $\beta = 2$ is particular because the subproblem in $W$ and $H$ are nonnegative least squares problems, that is, convex quadratic problems with Lipschitz continuous gradient. 
Therefore, highly efficient schemes exist when $\beta = 2$ that outperform the MU; for example exact block coordinate descent methods~\cite{Cichocki07HALS, gillis2012accelerated, kim2014algorithms, acc_NMF}, or fast gradient methods~\cite{guan2012nenmf, le2020inertial}. 
In this paper, we focus on the case $\beta < 2$.

In many applications, on top of the nonnegative constraints on the variables, additional constraints are needed to provide a meaningful solution. 
An instrumental example is the constraint that the entries in each column of $H$ sum to one; this is the so-called sum-to-one constraint that is crucial in blind hyperspectral unmixing; see Section~\ref{sec_ssbetaNMF}. 
Another example is a sum-to-one constraint on the columns of $W$ along with a volume regularizer on $W$. This model leads to identifiability of the factors $W$ and $H$ under mild conditions; see Section~\ref{sec_minvolKLNMF}. 
Most algorithms that deal with such equality constraints do it a posteriori with a projection onto the feasible set, or with a renormalization of the columns of $W$ and the rows of $H$ (that is, replace $W(:,k)$ and $H(k,:)$ by $\alpha_k W(:,k)$ and $H(k,:)/\alpha_k$ for some $\alpha_k > 0$),
so that their product $WH$ remains unchanged, and hence $D(V|WH)$ remains unchanged. 
Such approaches are not ideal: 
\begin{itemize}
    \item Projection requires to perform a line-search to ensure the monotonicity of the algorithm, that is, to ensure that the objective does not increase after each iteration, 
    which may be computationally heavy. 
    
    \item  Renormalization of the columns of $W$ and the rows of $H$ is only useful when each constraint applies to a column of $W$ {or} a row of $H$. 
    It is not applicable for example for the sum-to-one constraint on the columns of $H$ mentioned above. Moreover, in the presence of regularization terms in the objective function, it may destroy the monotonicity of the algorithm. 
    
\end{itemize} 
Another approach is to use parametrization. However, as far as we know, it does not guarantee the monotonicity of the algorithm; see Section~\ref{sec_ssbetaNMF} for more details.

\paragraph{Outline and contribution}

In this paper, we introduce a general framework to design MU for 
$\beta$-NMF with disjoint linear equality constraints, and with penalty terms in the objective function. By disjoint, we mean that each variable appears in at most one equality constraint. 
This framework, presented in Section~\ref{sec_generalframework}, does not resort to projection, renormalization, or parametrization.  
Our MU satisfy the set of constraints after each update of the variables during the optimization process, while guaranteeing that the objective function decreases monotonically. 
This framework works as follows:  
\begin{itemize}
    \item First, as for the standard MU for $\beta$-NMF, we majorize the objective function using a separable majorizer, that is, the majorizer is the sum of functions involving a single variable. 
    
    \item Second, we construct the augmented Lagrangian for the majorizer. Because the majorizer is separable, the problem can be decomposed into independent subproblems involving only variables that occur in the same equality constraint since they are  disjoint.   
    For a fixed value of the Lagrange multipliers, we prove that the solution of these subproblems are unique, under mild conditions (Proposition~\ref{prop:mingmu}). Moreover, it can be written in closed form via MU for specific values of $\beta$ and depending on the regularizer used (this is summarized in Table~\ref{tab:poldeg}). 
    
    \item Finally, we prove that, under mild conditions, there is a unique solution for the Lagrange multipliers so that the equality constraints are satisfied (Proposition~\ref{prop:mu_unicity}). This allows us to apply the Newton-Raphson method to compute the Lagrange multipliers while guaranteeing quadratic convergence (Proposition~\ref{NR}).

\end{itemize}

We then showcase this framework on two NMF models, and show that it competes favorably with the state of the art: 
\begin{enumerate}
    \item A $\beta$-NMF model with sum-to-one constraints on the columns of $H$, which we refer to as simplex-structured $\beta$-NMF (Section~\ref{sec_ssbetaNMF}), and 
    
    \item A minimum-volume $\beta$-NMF model with sum-to-one constraints on the columns of $W$ (Section~\ref{sec_minvolKLNMF}).  
\end{enumerate} 
Finally, Section~\ref{sec_nonlinear} shows that the framework can be extended to the case of quadratic disjoints constraints, which we showcase on 
sparse $\beta$-NMF with $\ell_2$-norm constraints on the columns of $W$.




\section{General framework to design MU for $\beta$-NMF under disjoint \ji{linear} equality constraints and penalization}
\label{sec_generalframework}
 



In this paper, we introduce a general framework to tackle $\beta$-NMF with disjoint linear equality constraints, and with penalty terms in the objective function. Let us first introduce specific notations: given a matrix $A \in \mathbb{R}^{F \times N}$ and a list of indices $\mathcal{K} \subseteq \{ (f,n) \ | \ 1 \leq f \leq F, 1 \leq n \leq N\}$, we denote by $A(\mathcal{K})$ the vector of dimension $|\mathcal{K}|$ whose entries are the entries of $A$ corresponding to the indices within $\mathcal{K}$. 
Let us introduce $\mathcal{K}_{i}$ ($1 \leq i \leq I$) and $\mathcal{B}_{j}$ ($1 \leq j \leq J$) to be disjoint sets of indices for the entries of $W$ and $H$, respectively, that is, 
\begin{itemize}

    \item $\mathcal{K}_{i} \subseteq \{ (f,k) \ | \ 1 \leq f \leq F, 1 \leq k \leq K\}$ for  $i = 1,2,\dots,I$, 
    
    \item $\mathcal{B}_{j} \subseteq \{ (k,n) \ | \ 1 \leq k \leq K, 1 \leq n \leq N\}$ for $j = 1,2,\dots,J$,

    \item $\mathcal{K}_{u} \cap \mathcal{K}_{v} = \emptyset$ for all $1 \leq u,v \leq I$ and $u \neq v$,
    
    \item $\mathcal{B}_{p} \cap \mathcal{B}_{q} = \emptyset$ for all $1 \leq p,q \leq J$ and $p \neq q$. 
\end{itemize} 
We now define penalized $\beta$-NMF with disjoint linear equality constraints as follows 
\begin{equation}\label{eq:betaNMFmodel}
\begin{aligned}
& \underset{W\in \mathbb{R}^{F \times K}_+,H\in \mathbb{R}^{K \times N}_+}{\min}
& & D_{\beta}\left(V|WH\right) + \lambda_1 \Phi_1(W) + \lambda_2 \Phi_2(H) \\
& \text{such that}
& & \alphab_i^{T} W \left( \mathcal{K}_{i} \right) = b_i \text{ for } 1\leq i \leq I, \\
& \text{}
& &  \gammab_j^{T} H \left( \mathcal{B}_{j} \right) = c_j \text{ for } 1\leq j \leq J,
\end{aligned}
\end{equation}
where 
\begin{itemize}
    \item the penalty functions $\Phi_1(W)$ and $\Phi_2(H)$ 
    are lower bounded  and admit a particular upper approximation;  see Assumption~\ref{ass:phi} below. 
    
    \item $\lambda_1$ and $\lambda_2$ are the penalty weights (nonnegative scalars). 
    
    \item $\alphab_i \in \mathbb{R}^{|\mathcal{K}_i|}_{++}$ ($1\leq i \leq I$) and $\gammab_j \in \mathbb{R}^{|\mathcal{B}_j|}_{++}$ ($1\leq j \leq J$) are vectors with positive entries. 
    Note that if $\alphab_i$ or $\gammab_j$ contains zero entries,  the corresponding indices can be removed from $\mathcal{K}_{i}$ and $\mathcal{B}_{j}$.  
    
    \item $b_i$ ($1\leq i \leq I$) and $c_j$  ($1\leq j \leq J$) are positive scalars. 
\end{itemize}

As for most NMF algorithms, we propose to resort to a  block coordinate descent (BCD) framework 
to solve problem~\eqref{eq:betaNMFmodel}: at each iteration we tackle two sub-problems separately; one in $W$ and the other in $H$. 
The subproblems in $W$ and $H$ are essentially the same, by symmetry of the model, since transposing the relation $X \approx WH$ gives $X^T \approx H^T W^T$. Hence, we may focus on solving the  subproblem in $H$ only, namely 
\begin{equation}
\label{eq:betaNMFmodelH}
 \underset{H\in \mathbb{R}^{K \times N}_+}{\min}
 D_{\beta}\left(V|WH\right) + \lambda_2 \Phi_2(H) 
\quad \text{ such that } \quad 
 \gammab_j^{T} H \left( \mathcal{B}_{j} \right)=c_j \text{ for } 1\leq j \leq J. 
\end{equation}


In order to solve~\eqref{eq:betaNMFmodelH}, we will design MU based on the majorization-minimization (MM) framework~\cite{sun2017majorization}, which is the standard in the NMF literature; see~\cite{Fevotte_betadiv} and the references therein. 
Let us briefly recall the high-level ideas to obtain MU via MM. 
Let us consider the general  problem 
\[
\min_{h \in \mathcal{H}} f(h). 
\] 
Given an initial iterate $\widetilde{h}\in \mathcal{H}$, MM generates a new iterate $\hat{h}\in \mathcal{H}$ that is guaranteed to decrease the objective function, that is,
$f\big( \hat{h} \big) \leq f\big( \widetilde{h}  \big)$. To do so, it uses the following two steps: 
\begin{itemize} 
\item Majorization: find a function that is an upper approximation of the objective and is tight at the current iterate, which is referred to as a majorizer. More precisely find a function $g\big(h|\widetilde{h}\big)$ such that 
\[
(i) ~ g\big(\widetilde{h}|\widetilde{h}\big) = f\big(\widetilde{h}\big) 
\quad \text{ and } \quad 
(ii)  ~  g\big({h}|\widetilde{h}\big) \geq f({h}) \text{ for all } h \in  \mathcal{H}. 
\]

\item Minimization: minimize the majorizer, that is, 
solve $\min_{h \in \mathcal{H}} g\big({h}|\widetilde{h}\big)$ approximately or exactly, to obtain the next iterate $\hat{h} \in \mathcal{H}$ which is such that $(iii) ~ g\big(\hat{h}|\widetilde{h}\big)  \leq g(\widetilde{h}|\widetilde{h})$. 
This guarantees the objective function to decrease at each step of this iterative process since 
\[
f\big(\hat{h}\big) 
\; 
\underset{(ii)}{\leq} 
\; 
g\big(\hat{h}|\widetilde{h}\big) 
\; 
\underset{(iii)}{\leq}  
\; 
g\big(\widetilde{h}|\widetilde{h}\big) 
\; 
\underset{(i)}{=} 
\; 
f\big(\widetilde{h}\big). 
\] 
\end{itemize}  
The MU for NMF are obtained using MM where the majorizer $g$ is chosen separable, that is, $g\big(h|\widetilde{h}\big) = \sum_{i=1} g_i\big(h_i|\widetilde{h}_i\big)$ for some well chosen univariate functions $g_i$'s; see~\eqref{eqn:auxbeta} in the next section. 
This choice typically makes the minimization of $g$ admits a closed-form solution which is multiplicative, that is, it has the form 
$\hat{h} = \widetilde{h} \odot c\big(\widetilde{h}\big)$ where $\odot$ is the component-wise product, and $c\big(\widetilde{h}\big)$ is a nonnegative vector that  depends on $\widetilde{h}$. We will encounter several examples later in this paper.


In summary, to derive MU for~\eqref{eq:betaNMFmodelH}, we will follow the MM framework. 
\color{black}  
We first provide a majorizer for the objective of~\eqref{eq:betaNMFmodelH} in Section~\ref{sec:majorizer}. This majorizer has the property to be separable in each entry of $H$. In order to handle the equality constraints, 
we introduce Lagrange dual variables in Section~\ref{sec:lagrange}, and explain how they can be computed efficiently. 
This allows us to derive general MU in Section~\ref{betanmf_disjointconstraint} in the case of non-penalized $\beta$-NMF under disjoint linear equality constraints.
 This is showcased on simplex-structured $\beta$-NMF in Section~\ref{sec_ssbetaNMF}. In Section~\ref{sec_minvolKLNMF}, we will illustrate on minimum-volume KL-NMF  how to derive MU in the presence of penalty terms.


\ngi{\subsection{Separable majorizer for the objective function}}
\label{sec:majorizer} 

Let us derive a majorizer for 
$\Psi(H) := D_{\beta}\left(V|WH\right) + \lambda \Phi(H)$, 
that is, a function $G\big(H|\widetilde{H}\big)$ satisfying 
(i) $G\big(H|\widetilde{H}\big)  \geq \Psi(H)$ for all $H$, and 
(ii)  $G\big(\widetilde{H}|\widetilde{H} \big)  = \Psi\big(\widetilde{H} \big)$. Note that, to simplify the presentation, we denote $\Phi_2(H) = \Phi(H)$ and $\lambda = \lambda_2$. 
To do so, let us analyze each term of $\Psi(H)$ independently. 

\paragraph{Majorizing $D_{\beta}\left(V|WH\right)$} The first term $D_{\beta}\left(V|WH\right)$ can be decoupled into $N$ independent terms, one for each column $\hb_n$ of $H$, that is, 
$D_{\beta}\left(V|WH\right) = \sum_{n=1}^N D_{\beta}\left(\vb_n|W\hb_n\right)$, 
where $\vb_n$ denotes the $n$th column of matrix $V$. 
Let us focus on a specific column of $H$, denoted $\hb \in \mathbb{R}^K_+$, and the corresponding column of $V$, denoted $\vb \in \mathbb{R}^K_+$.  
We majorize $D_{\beta}(\vb|W\hb)=\sum_{f=1}^F d_{\beta}(v_f|(W\hb)_f)$  following the methodology introduced in \cite{Fevotte_betadiv}, {which consists in applying a convex-concave procedure \cite{Yuille03} to $d_{\beta}$}, as presented in Appendix~\ref{convconcacm}. The resulting upper bound is given by
\begin{equation}
\label{eqn:auxbeta}
d_{\beta}(v_f|(W\hb)_f) \leq
\sum_{k=1}^K \frac{w_{fk}\widetilde{h}_k}{\widetilde{v}_f}\widecheck{d} \left( v_f |  \widetilde{v}_f \frac{h_k}{\widetilde{h}_k} \right)
+ \widehat{d}'\big(v_f|\widetilde{v}_f\big) \sum_{k=1}^K w_{fk} \big(h_k - \widetilde{h}_k\big)
+ \widehat{d}\big(v_f | \widetilde{v}_f\big), 
\end{equation}
where $w_{fk}$ denotes the entry of matrix $W$ at position $(f,k)$, 
$\widetilde{v}_f := \big(W \widetilde{\hb}\big)_f$ denotes the $f$th entry of $\widetilde{\vb}$, 
and 
$\widehat{d}$ and $\widecheck{d}$ are the concave and convex parts of $d$, respectively. 

\paragraph{Majorizing $\Phi(H)$}  For the second term $\Phi(H)$, we rely on the following assumption for $\Phi$. 
\begin{assumption} \label{ass:phi}
The function 
$\Phi : \mathbb{R}^{K \times N}_+ \mapsto \mathbb{R}$ 
{is lower bounded, and} 
for any $\widetilde{H} \in \mathbb{R}^{K \times N}_+$ there exists constants $L_{kn}$ ($1 \leq k \leq K, 1 \leq n \leq N$) such that  the inequality 
\begin{equation} \label{eq:Lipschtiz_case} 
\Phi(H) 
  \;  \leq \;    \Phi\big( \widetilde{H} \big) + \left\langle \nabla\Phi(\widetilde{H}), H - \widetilde{H} \right\rangle + \sum_{k,n} \frac{L_{kn}}{2} (H_{kn} - \widetilde{H}_{kn})^2 
\end{equation} 
  is satisfied for all $H \in \mathbb{R}^{K \times N}_+$.  (Note that the constants $L_{kn}$ may depend on $\widetilde{H}$, this will be the case for example in Section~\ref{sec_minvolKLNMF}\color{black}).   
\end{assumption}

Let us mention two important classes of functions satisfying Assumption~\ref{ass:phi}. 
\begin{enumerate}

\item Smooth concave functions that are lower bounded on the nonnegative orthant. For such functions, we can take $L_{kn} = 0$ for all $k,n$ since they are upper approximated by their first-order Taylor approximation.  
Note that, in this case, 
\begin{equation}
\label{eq:geq0}
\nabla\Phi(\widetilde{H}) \geq 0,
\end{equation}
otherwise we would have 
$\lim_{y\rightarrow\infty}\Phi\big(H + y \eb_i \eb_j^T \big)=-\infty$, where $\eb_i$ is the $i$th unit vector, and 
this would contradict the fact that $\Phi$ is bounded from below. 
This observation will be useful in  the proof of Proposition~\ref{prop:mingmu} and is only valid for the special case $L_{kn} = 0$ for all $k,n$.  

Examples of such penalty functions include the  sparsity-promoting regularizers $\Phi(H) = \| H \|_p^p =  \sum_{k,n} H(k,n)^p$ for $0 < p \leq 1$ since $H \geq 0$.

\item Lower-bounded functions with Lipschitz continuous gradient for which~\eqref{eq:Lipschtiz_case} follows from the descent lemma~\cite{Bertsekas99}. 

Examples of such penalty functions include any smooth convex functions; for example any quadratic penalty, such as  
$||A H - B ||_2^2$ for some matrices $A$ and $B$ in which case $L_{kn} = \sigma_1(A)^2$ for all $k,n$.  
We will encounter another example later in the paper, namely $\logdet\big(HH^\top + \delta I\big)$ for $\delta > 0$ which allows to minimize the volume of the rows of $H$; see Section~\ref{sec_minvolKLNMF} for the details  
(Note that we will use this regularizer for $W$). 

\end{enumerate}


\paragraph{Majorizing $\Psi(H)$} 

Combining \eqref{eqn:auxbeta} and \eqref{eq:Lipschtiz_case}, we can construct a majorizer for $\Psi(H)$. Since both \eqref{eqn:auxbeta} and \eqref{eq:Lipschtiz_case} are separable in each entry of $H$, 
their combination is also separable into a sum of $K\times N$ component-wise majorizers, up to an additive constant: 
\begin{equation}
G\big(H|\widetilde{H}\big) 
= \sum_{n=1}^N\sum_{k=1}^K g\big(h_{kn}|\widetilde{H}\big) + C\big(\widetilde{H}\big),
\end{equation}
where
\begin{align}
\label{eq:Gkn}
g\big(h_{kn}|\widetilde{H}\big) &= \sum_{f=1}^F \frac{w_{fk}\widetilde{h}_{kn}}{\widetilde{v}_{fn}}\widecheck{d} \left( v_{fn} |  \widetilde{v}_{fn} \frac{h_{kn}}{\widetilde{h}_{kn}} \right)
+ a_{kn} h_{kn}^2 + p_{kn} h_{kn},
\\
C\big(\widetilde{H}\big) &= \sum_{n=1}^N\sum_{f=1}^F\left(\widehat{d}\big(v_{fn} | \widetilde{v}_{fn}\big)
- \sum_{k=1}^K\widehat{d}'\big(v_{fn}|\widetilde{v}_{fn}\big) w_{fk} \widetilde{h}_{kn}\right) \ngi{+ a_{kn} \widetilde{h}_{kn}^2},
\notag
\end{align}
with $a_{kn}=\lambda\frac{L_{kn}}{2}$, and
\begin{equation*}
p_{kn}=\sum_{f=1}^F w_{fk}\widehat{d}'\big(v_{fn}|\widetilde{v}_{fn}\big) +\lambda\left(\frac{\partial \Phi}{\partial h_{kn}}\big(\widetilde{H}\big)- L_{kn} \widetilde{h}_{kn}\right).
\end{equation*} 

\subsection{Dealing with equality constraints via Lagrange dual variables}  \label{sec:lagrange} 

In the previous section, we derived a majorizer for $\Psi(H)$, $G\big(H|\widetilde{H}\big)$, which is separable in each entry of $H$. Without the equality constraints, we could then compute closed-form solutions to univariate problems to minimize $G\big(H|\tilde{H}\big)$ to obtain the standard MU for NMF as in~\cite{Fevotte_betadiv}. 

However, in problem~\eqref{eq:betaNMFmodelH}, the entries of $H$ in the subsets $\mathcal{B}_j$ are not independent as they are linked with the equality constraints $\gammab_j^{T} H ( \mathcal{B}_{j})=c_j$ for $j = 1,2,\dots,J$. 
In fact, to minimize the majorizer under the equality constraints, we need to solve 
\begin{equation}
\label{eq:betaNMFmodelHmajor}
\underset{H\in \mathbb{R}^{K \times N}_+}{\min}
 G\big(H|\widetilde{H}\big) 
\quad  \text{ such that } \quad 
\gammab_j^{T} H (\mathcal{B}_{j})=c_j \text{ for } 1\leq j \leq J. 
\end{equation} 
The variables in different sets $\mathcal{B}_j$ can be optimized independently, as they do not interact in the majorizer nor in the constraints. Note that, for the entries of $H$ that do not appear in any constraints, the standard MU~\cite{Fevotte_betadiv} can be used. 
For simplicity, let us fix $j$ and 
denote $\mathcal{B}=\mathcal{B}_j$, $Q = |\mathcal{B}|$, 
$\yb = H(\mathcal{B}) \in \mathbb{R}^{Q}_{+}$, $\gammab = \gammab_j \in \mathbb{R}^{Q}_{++}$, 
and $c = c_j > 0$.  
The problems we need to solve have the form 
\begin{equation}
\label{eq:minG}
\min_{\yb \in \Yc} G\big(\yb | \widetilde{H} \big), 
\end{equation} 
where $\Yc=\left\{\yb \in \mathbb{R}^{Q}_+ \ | \ \gammab^T\yb=c \right\}$ and
\begin{equation}
\label{eq:Gy}
G\big(\yb | \widetilde{H}\big)=\sum_{(k,n)\in \mathcal{B}}g\big(h_{kn} | \widetilde{H} \big),
\end{equation} 
where the component-wise majorizers $g\big(h_{kn}|\widetilde{H}\big)$ are defined by~\eqref{eq:Gkn}. 
Let us introduce a convenient notation: for ${q} = 1,2,\dots,Q$, we denote by $(k({q}), n({q}))$ the $q$th pair belonging to  $\mathcal{B}$. 
Hence the Lagrangian function of~\eqref{eq:Gy} can be written as 
%
%
{
\begin{equation}
\label{eq:Gmuy}
    G^\mu\big(\yb|\widetilde{H}\big)  
    = G\big(\yb|\widetilde{H}\big)-\mu(\gammab^T\yb-c) 
    = \mu c + C\big(\widetilde{H}\big) + \sum_{{q}=1}^Q  g^\mu\big(y_{q}|\widetilde{H}\big), 
\end{equation}
where
\begin{align}
g^\mu\big(y_{q}|\widetilde{H}\big) 
& = g\big(y_{q}|\widetilde{H}\big)-\mu \gamma_{q}\,y_{q}\notag \\ 
& = \sum_{f=1}^F \frac{w_{fk({q})}\widetilde{y}_{q}}{\widetilde{v}_{fn({q})}}\widecheck{d} \left( v_{fn({q})} |  \widetilde{v}_{fn({q})} \frac{y_{q}}{\widetilde{y}_{q}} \right)
+ a_q y_{q}^2 + (p_{q}-\mu \gamma_{q}) y_{q},
\label{eq:gmu}
\\
p_{q}&=\sum_{f=1}^F w_{fk({q})}\widehat{d}'\big(v_{fn({q})}|\widetilde{v}_{fn({q})}\big) +\lambda\left(\frac{\partial \Phi}{\partial y_{q}}\big(\widetilde{H}\big)- L_{k(q)n(q)} \widetilde{y}_{q}\right), 
\label{eq:pell}
\end{align}
}
and $\mu\in\mathbb{R}$.
Note that $G^\mu$ is separable, as is $G$, because the term $\gammab^T \yb$ is linear. 

Assume for now that the Lagrangian multiplier $\mu$ is known, and let us minimize $G^\mu\big(\yb|\widetilde{H}\big)$ on $(0,\infty)^Q$. 
Such a problem is separable under the form of $Q$ subproblems, consisting in minimizing univariate functions  $g^\mu\big(\cdot|\widetilde{H}\big)$ separately over $(0,\infty)$. We now show in Proposition~\ref{prop:mingmu} that, under mild conditions, each subproblem admits a unique solution over $(0,\infty)$. 
%
\begin{proposition}
\label{prop:mingmu}
Let ${q} \in \{1,2,\dots,Q\}$. 
Assume that $\beta<2$ and $\widetilde{y}_{q},v_{fn({q})},w_{fk({q})}>0$ for all $f$.
Moreover, when $\beta\leq1$, assume that $\mu < \frac{p_{q}}{\gamma_{q}}$ for all $q$ such that $a_q=0$. 
Then there exists a unique minimizer $y_{q}^\star(\mu)$ of $g^\mu\big(y_{q}|\widetilde{H}\big)$ in $(0,\infty)$.  
\end{proposition}

%
\begin{proof}
According to Proposition~\ref{prop:approxdbeta} (see Appendix~\ref{convconcacm}), each $g^\mu$ is $C^\infty$ and strictly convex on $(0,\infty)$, so its infimum is uniquely attained in the closure of $(0,\infty)$. We have to prove that it is neither reached at $0$ nor at $\infty$. 
On the one hand, \ji{from \eqref{eq:gmu}, we have 
\begin{equation}
\label{eqgmu'}
(g^\mu)'\big(y_{q}|\widetilde{H}\big)=\sum_{f=1}^F w_{fk({q})}\widecheck{d}' \left( v_{fn({q})} | \widetilde{v}_{fn({q})}\frac{y_{q}}{\widetilde{y}_{q}} \right) +2a_q y_{q} + p_{q} -\gamma_{q}\mu
\end{equation}
}%
and, for any $\beta<2$ and any $x>0$,
\begin{equation*}
\lim_{y\rightarrow0^+}\widecheck{d}'(x|y)=-\infty,\\
\end{equation*}
so $\lim_{y_{q}\rightarrow0^+}(g^\mu)'\big(y_{q}|\widetilde{H}\big)=-\infty$, which ensures that the infimum is not reached at $0$.
On the other hand,
\begin{equation}
\lim_{y\rightarrow\infty}\widecheck{d}'(x|y)=\begin{cases}0&\text{if } \beta\leq1,\\\infty&\text{otherwise.}\end{cases}
\label{eq:d'inty}
\end{equation}
According to \eqref{eqgmu'} and \eqref{eq:d'inty}, the distinction must be made between two cases:
\begin{itemize}
\item If $a_q>0$ or $\beta\in(1,2)$: $\lim_{y_{q}\rightarrow\infty}(g^\mu)'\big(y_{q}|\widetilde{H}\big)=\infty$, so the infimum is reached for a finite $y_{q}$. 
\item If $a_q=0$ and $\beta\leq1$: $\lim_{y_{q}\rightarrow\infty}(g^\mu)'\big(y_{q}|\widetilde{H}\big)=p_{q} -\gamma_{q}\mu$, so the same conclusion holds if $\mu<\frac{p_{q}}{\gamma_{q}}$. 
\end{itemize}
\end{proof}
We just proved that, under mild conditions, each $g^\mu$ has a unique minimizer over $(0,\infty)$.
However we assumed that the value of $\mu$ is fixed.  
Now given $\yb^\star(\mu)=\big[y_1^\star(\mu),\ldots,y_Q^\star(\mu)\big]^T$, let us show 
that the solution to $\gammab^T \yb^\star(\mu) =c$ is unique. The corresponding value of $\mu$, which we denote  $\mu^\star$, provides  the minimizer $\yb^\star(\mu^\star)$ of $G^\mu(\yb|\widetilde{H})$ that satisfies the linear constraint $\gammab^T\yb^\star(\mu^\star)=c$. \ji{Moreover, $\mu^\star$ naturally fulfills $\mu^\star<\frac{p_{q}}{\gamma_{q}}$ for all $q$ when $\beta\leq1$ and $a_q=0$, as required in Proposition~\ref{prop:mingmu}.}
%
\begin{proposition}
\label{prop:mu_unicity} 
Assume that $\beta<2$ and $\widetilde{y}_{q}, v_{fn({q})},w_{fk({q})}>0$ for all ${q},f$. 
Then the scalar equation $\gammab^T \yb^\star(\mu) =c$ in the variable $\mu$ admits a unique solution $\mu^\star$ in $(-\infty,t)$, where
\begin{align}
\label{eq:t}
t&=\min_{1 \leq q \leq Q} t_q, \text{ where }
t_q=\begin{cases} \frac{p_{q}}{\gamma_{q}}&\text{if $\beta\leq1$ and $a_q=0$,}\\\infty&\text{otherwise, }\end{cases}
\end{align} 
{so that $\yb^\star(\mu^\star)$ $\in(0,\infty)^Q$} is the unique solution to problem~\eqref{eq:minG}.
\end{proposition}
\begin{proof}
Under the conditions of Proposition~\ref{prop:mingmu}, $g^\mu\big(y_{q}|\widetilde{H}\big)$ has a unique minimizer $y^\star_{q}(\mu)$ for each $j$.  
By the first-order optimality condition, $y^\star_{q}(\mu)$ is a solution of $(g^\mu)'\big(y_{q}|\widetilde{H}\big)=0$ or equivalently, by \eqref{eqgmu'}, a solution of $\gamma_{q}^{-1}g'\big(y_{q}|\widetilde{H}\big)=\mu$ over $(0,\infty)$ where 
\begin{equation}
\label{eq:g'}
\gamma_{q}^{-1}g'\big(y_{q}|\widetilde{H}\big)=\gamma_{q}^{-1}\sum_{f=1}^F w_{fk({q})}\widecheck{d}' \left( v_{fn({q})} | \widetilde{v}_{fn({q})} \frac{y_{q}}{\widetilde{y}_{q}} \right) +2 \frac{a_q}{\gamma_{q}}y_{q} + \frac{p_{q}}{\gamma_{q}}
\end{equation}
is strictly increasing on $(0,\infty)$ (since $g$ is strictly convex) and one-to-one from $(0,\infty)$ to an open interval
$T_{q}=(t^-_{q},t^+_{q})$ where
\begin{align}
t^-_{q}&=\lim_{y_{q}\rightarrow0}g'\big(y_{q}|\widetilde{H}\big)=-\infty,\\
t^+_{q}&=\lim_{y_{q}\rightarrow\infty}g'\big(y_{q}|\widetilde{H}\big)
 = t_q. 
\end{align}
Moreover, $p_{q}\geq0$ if $a_q=0$ (then $L=0$) and $\beta\leq1$ according to~\eqref{eq:geq0} and \eqref{eq:pell}.
As a consequence, $\gamma_{q}^{-1}g'(y_{q}^\star|\widetilde{y}_{q})=\mu$ is equivalent to
\begin{equation}
\label{eq:yhat}
y_{q}^\star(\mu) = \big(g'\big)^{-1}(\gamma_{q}\mu),
\end{equation}
where $\mu\in T_{q}$ and $\big(g'\big)^{-1}$ denotes the inverse function of $g'$.

Coming back to the multivariate problem~\eqref{eq:minG}, we must find a value $\mu^\star$ of the Lagrangian multiplier such that the constraint $\gammab^T\yb^\star(\mu)=c$ is satisfied. Given~\eqref{eq:yhat}, $\mu^\star$ is a solution of
\begin{equation}
\label{eq:muhat}
\sum_{{q}=1}^Q \gamma_{q}\big(g'\big)^{-1}(\gamma_{q}\mu)=c.
\end{equation}
Each $g'\big(y_{q}|\widetilde{H}\big)$ being strictly increasing on $(0,\infty)$, $(g')^{-1}(\gamma_{q}\mu)$ is also strictly increasing (from $T_{q}$ to $(0,\infty)$), this is a direct consequence of $(f^{-1})'=\frac1{f' \circ f^{-1}}$ where $f$ is any strictly increasing function on some interval. Finally $\sum_{{q}=1}^Q  \gamma_{q}(g')^{-1}(\gamma_{q}\mu)$ is strictly increasing from $\cap_{j=1}^J T_{q}=(-\infty,t)$ to $(0,\infty)$, with $t\geq0$. Therefore, the solution $\mu^\star$ is unique. 
\end{proof}

 Proposition~\ref{prop:mu_unicity} shows that the optimal Lagrangian multiplier is the unique solution of~\eqref{eq:muhat}. Finding the solution of~\eqref{eq:muhat} is equivalent to finding the root of a function $r(\mu)$. We propose here-under to use a Newton-Raphson method to compute $\mu^\star$, and show that this method generates a sequence of iterates $\mu_n$ that converges towards $\mu^\star$ at a quadratic speed.

\begin{proposition}
\label{NR} 
Assume that $\beta<2$ and $\widetilde{y}_{q},v_{fn({q})},w_{fk({q})}>0$ for all ${q},f$. 
Let
$$
r(\mu)=\sum_{{q}=1}^Q  \gamma_{q}(g')^{-1}\big(\gamma_{q}\mu\big)-c
$$
for $\mu\in(-\infty,t)$ where $t$ is defined in~\eqref{eq:t}, and denote $\mu^\star$ the unique solution of $r(\mu)=0$. 
From any initial point $\mu_0\in(\mu^\star,t)$, Newton-Raphson's iterates
$$
\mu_{n+1}=\mu_n-\frac{r(\mu_n)}{r'(\mu_n)}
$$
decrease towards $\mu^\star$ at a quadratic speed.
\end{proposition}
\begin{proof}
We already know that $r$ is strictly increasing from $(-\infty,t)$ to $(0,\infty)$. Let us show that $r$ is also strictly convex.
According to the third item of Proposition~\ref{prop:approxdbeta} in Appendix~\ref{convconcacm}, $\widecheck{d}''(x|y)$ is completely monotonic, so it is strictly decreasing in $y$. Equivalently, $\widecheck{d}'(x|y)$ is strictly concave in $y$, and each $g'$ is also strictly concave according to~\eqref{eq:g'}. Since the inverse of a strictly increasing, strictly concave function $f$ is strictly increasing and strictly convex, which is a direct consequence of $(f^{-1})''=-\frac{f''\circ f^{-1}}{(f' \circ f^{-1})^3}$, 
then each $(g')^{-1}$ is strictly convex, and finally, $r$ is strictly convex.

For any $\mu_0\in(\mu^\star,t)$, we have $r(\mu_0)>0$, so $\mu_1=\mu_0-\frac{r(\mu_0)}{r'(\mu_0)}<\mu_0$. We have also $\mu_1>\mu^\star$ as a consequence of the strict convexity of $r$. By immediate recurrence, we obtain that $\mu_n$ is a decreasing series that converges towards $\mu^\star$. According to \cite{Ortega70}, it converges at a quadratic speed since $|r'|$ and $|r''|$ are bounded away from 0 in $[\mu^\star,\mu_0]$.
\end{proof}

\paragraph{Discussion} 

At this point, we have derived an optimization framework to tackle problem~\eqref{eq:minG}. 
The optimal Lagrangian multiplier value is determined before each majorization-minimization update using a Newton-Raphson algorithm. However, such a formal solution is implementable if and only if each $y_{q}^\star(\mu)$ can be actually computed as the minimizer of $g^\mu\big(y_{q}|\widetilde{H}\big)$ in $(0,\infty)$. In some cases, computing $y_{q}^\star(\mu)$ is equivalent to extracting the roots of a polynomial of a degree smaller or equal to four, which is possible in closed form. In other cases, we have to solve a polynomial equation of degree larger than four, or even an equation that is not polynomial. Table~\ref{tab:poldeg} indicates the cases where a closed-form solution is available, and hence when our framework can be efficiently implemented. 
\begin{table}[h]
\centering
\setlength{\tabcolsep}{3pt}
\begin{tabular}{|c|c|c|c|c|c|c|c|}
\hline
\hfill     
&$\beta \in (-\infty,1)\setminus\{0\}$ 
& $\beta = 0$ 
& $\beta = 1$ & 
\multicolumn{4}{|c|}{$\beta \in (1,2)$}      \\
 &    &   &   &$\frac54$&$\frac43$&$\frac32$&other           \\\hline
\begin{tabular}{c}
\revise{No penalization, or} \\ 
\revise{{$L_{kn} = 0$ for all $k,n$}} 
\end{tabular} 
&1&1&1&3&4&2&$\rip$\\\hline
{$L_{kn} > 0$ for some $k,n$}  &$\rip$&3&2&$\rip$&$\rip$&3&$\rip$\\\hline
\end{tabular}
\caption{
Cases where \eqref{eq:yhat} can be computed in closed form. They are indicated by the degree of the corresponding polynomial equation, otherwise the symbol $\rip$ is used. The constants $L_{kn}$ are the one needed in Assumption~\ref{ass:phi} for the penalization functions $\Phi_1(H)$ and $\Phi_2(W)$; see~\eqref{eq:Lipschtiz_case}. }
\label{tab:poldeg}
\end{table}
\revise{We observe that, without penalization or with penalization satisfying $L_{kn} = 0$ for all $k,n$ (e.g., smooth concave functions), 
the polynomial equation is of degree one, and hence always admit a closed form for $\beta \leq 1$ 
and $\beta \in \left\{\frac{5}{4}, \frac{4}{3}, \frac{3}{2} \right\}$. 
} 
This particular case is discussed in the next section, which we will exemplify in Section~\ref{sec_ssbetaNMF} with $\beta$-NMF with sum-to-one constraints on the columns of $H$.  
In Section~\ref{sec_minvolKLNMF}, we will present an important example with $L_{kn} > 0$ for all $k,n$ and $\beta = 1$, namely minimum-volume KL-NMF.

\subsection{MU for $\beta$-NMF with disjoint linear equality constraints without penalization}\label{betanmf_disjointconstraint}

In this section, we derive an algorithm based on the general framework presented in the previous section to tackle the $\beta$-NMF problem under disjoint linear equality constraints without penalization, that is,  problem~\eqref{eq:betaNMFmodel} with $\lambda_1=\lambda_2=0$.
We consider this simplified case here as it allows to provide explicit MU for any value of $\beta < 2$; see the row `No penalization' of Table~\ref{tab:poldeg}. These updates  satisfy the constraints after each update of $W$ or $H$, and monotonically decrease the objective function $D_{\beta}\left(V|WH\right)$. 

Let us then consider the subproblem of \eqref{eq:betaNMFmodel} over $H$ when $W$ is fixed and with $\lambda_2=0$, that is,  
\begin{equation}\label{eq:subprobleminW}
 \underset{H\in \mathbb{R}^{K \times N}_+}{\min}
 D_{\beta}\left(V|WH\right) 
\quad \text{such that} \quad \gammab_j^{T} H (\mathcal{B}_{j})=c_j \text{ for } 1\leq j \leq J.
\end{equation}
Let us follow the framework presented above. First, an auxiliary function, which we denote $G(H| \widetilde{H})$, is constructed at the current iterate $\widetilde{H}$ so that it majorizes the objective for all $H$ and is defined as follows:
\begin{equation}
\label{eq:G}
\begin{aligned}
G\big(H|\widetilde{H}\big) 
& = \sum_{f,n}\left[\sum_{k} \frac{w_{fk}\widetilde{h}_{kn}}{\widetilde{v}_{fn}}\widecheck{d}\left(v_{fn} \Big|\widetilde{v}_{fn}\frac{h_{kn}}{\widetilde{h}_{kn}}\right)\right] +\left[\widehat{d}'\big(v_{fn}|\widetilde{v}_{fn}\big)\sum_{k}w_{fk}\big(h_{kn}-\widetilde{h}_{kn}\big)+\widehat{d}\big(v_{fn}|\widetilde{v}_{fn}\big) \right], 
\end{aligned}
\end{equation} 
where $\widecheck{d}(.|.)$ and $\widehat{d}(.|.)$ are  given in Appendix~\ref{convconcacm}.
Second, we need to minimize $G\big(H|\widetilde{H}\big)$ 
while imposing the set of linear constraints $\gammab_j^{T} H (\mathcal{B}_{j})=c_j$. The Lagrangian function of $G$ is given by 
\begin{equation}\label{eq:Gmu}
\begin{aligned}
G^{\mu}(H|\widetilde{H})= G(H |\widetilde{H} )-\sum_{j}^{J} \left[\mu_{j} \left( \gammab_j^{T} H (\mathcal{B}_{j})-c_j\right) \right], 
\end{aligned}
\end{equation} 
 where $\mu_{j}$ are the Lagrange multipliers associated to each linear constraint $\gammab_j^{T} H (\mathcal{B}_{j})=c_j$. We observe that $G^{\mu}$ in~\eqref{eq:Gmu} is a separable majorizer in the variables $H$ of the Lagrangian function $D_{\beta}\left(V|WH\right) - \sum_{j}^{J} \left[ \mu_{j} \left( \gammab_j^{T} H (\mathcal{B}_{j})-c_j\right) \right]$. Due to the disjoitness of each subset of variables $\mathcal{B}_{j}$ \eqref{eq:Gmu}, we only consider the optimization over one specific subset $\mathcal{B}_{j}$. 
The minimizer \eqref{eq:yhat} of $G^{\mu}(H\left( \mathcal{B}_{j} \right)|\widetilde{H}\left( \mathcal{B}_{j} \right))$ has the following component-wise expression:
\begin{equation}\label{eq:updateW}
\begin{aligned}
H^\star\left(\mathcal{B}_{j}\right)=\widetilde{H}\left(\mathcal{B}_{j}\right) \odot \left( \frac{\left[C\left(\mathcal{B}_{j}\right) \right]}{\left[ D\left(\mathcal{B}_{j}\right) - \mu_{j}  \gammab_j \right]} \right)^{. \eta \left( \beta \right)} , 
\end{aligned}
\end{equation}  
where $C=W^{T} \left( \left( WH \right)^{.\left( \beta-2 \right) } \odot V \right)$, $D=W^{T}\left( WH\right)^{.\left( \beta-1 \right) }$, $\eta(\beta)=\frac{1}{2-\beta}$ for $\beta \leq 1$, 
and $\eta(\beta)=\frac{1}{\beta-1}$ for $\beta \geq 2$ \cite[Table 2]{Fevotte_betadiv}, $A \odot B$ 
(resp.\@ $\frac{\left[ A \right]}{\left[ B \right]}$) is the Hadamard product (resp.\@ division) between $A$ and $B$,  
$A^{.\alpha}$ is the element-wise $\alpha$ exponent of $A$. 
\revise{
The case $\beta \in (1,2)$ is more difficult: 
we need to find a root of a function of the form 
$\mu + b x^{\beta-1} - c x^{\beta-2} = 0$. 
For example, for $\beta = \frac{3}{2}$, we have  
$\mu  + b x^{1/2} - c x^{-1/2} = 0$. 
Using  $y=\sqrt{x}$, and after simplifications, 
we obtain 
$\mu y + b y^{2}  - c = 0$ leading to the positive root 
$x = \left(\frac{\sqrt{\mu^2+4bc} -\mu}{2b}\right)^2$.  
} 

{According to Proposition~\ref{prop:mu_unicity}, \eqref{eq:updateW} is a well-defined update from $(0,\infty)^Q$ to itself, provided that $\mu_j$ is tuned to $\mu_j^\star$. This brings us a structural guarantee that $D\left(\mathcal{B}_j\right) - \mu_j^\star  \gammab_j$ cannot cancel.}

Finally, we need to \revise{evaluate $\mu_j^\star$, which is uniquely determined on some interval $\left(-\infty, t \right)$ according to Proposition~\ref{prop:mu_unicity}. This amounts to solve $\gammab_j^{T} H^\star \left( \mathcal{B}_{j} \right)=c_j$. When $\beta\not\in(1,2)$, this is equivalent} to find the root of the function 
\begin{equation} \label{eq:fmulincstr} 
r_j(\mu_j)=\sum_{{q}=1}^{Q} \gamma_{j,{q}} \left[ \widetilde{H}\left(\mathcal{B}_{j}\right) \odot \left( \frac{\left[C\left(\mathcal{B}_{j}\right) \right]}{\left[ D\left(\mathcal{B}_{j}\right) - \mu_{j}  \gammab_j \right]} \right)^{. \eta \left( \beta \right)} \right]_{{q}}  -c_j , 
\end{equation}
where $\left[A\right]_{{q}}$ denotes the ${q}$-th entry of expression $A$.  
Indeed, $r_j(\mu_j)$ is a finite sum of elementary rational functions of $\mu_j$ and each of them is an increasing, convex function in $\mu_{j}$ over $\left(-\infty, t_{{q}}\right)$ with  $t_{{q}}=\frac{D_{{q}}\left(\mathcal{B}_{j}\right)}{\gamma_{j,{q}}}$ for each $q$. It is even completely monotone for all $\mu$ in $\left(-\infty, t_{{q}}\right)$
because $\eta\left( \beta \right)>0$~\cite{miller2001}. As a consequence $r_j(\mu_j)$ is also a completely monotone, convex increasing function of $\mu_{j}$ in $\left(-\infty, t \right)$, where $t=\text{min}\left( t_{{q}} \right)$. Finally, we can easily show that the function $r_j(\mu_j)$ changes of sign on the interval $\left(-\infty, t \right)$ by computing two limits at the closure of the interval. As $\mu^\star \in (-\infty,t)$, the update \eqref{eq:updateW} is nonnegative.
To evaluate $\mu^\star$, we use a Newton-Raphson method, with any initial point $\mu_0\in(\mu^\star,t)$, with a quadratic rate of convergence as demonstrated in Proposition~\ref{NR}. 
Algorithm~\ref{consDisBetaNMF} summarizes our method to tackle~\eqref{eq:betaNMFmodel} for all the $\beta$-divergences, \jiter{$\beta\not\in(1,2)$} which we refer to as disjoint-constrained $\beta$-NMF algorithm. The update for matrix $W$ can be derived in the same way, by symmetry of the problem. \revise{For $\beta\in(1,2)$, a case-by-case analysis could be carried out for the values of $\beta$ for which the minimizer of \eqref{eq:Gmu} takes a closed-form expression.}


 \begin{remark} 
 As noted above, the denominators of \eqref{eq:updateW} and \eqref{eq:fmulincstr} will be different from zero. This follows notably from our assumption 
 that $(W,H)>0$; see Propositions~\ref{prop:mingmu},~\ref{prop:mu_unicity} and~\ref{NR}. 
 This is a standard assumption in the NMF literature: the entries of $(W,H)$ are initialized with positive entries which ensures all iterates to remain positive. This is important because the MU cannot change an entry equal to zero~\cite{lin2007projected}; 
 this is the so-called zero-locking phenomenon.   This implies $C$ and $D$ in~\eqref{eq:updateW} and~\eqref{eq:fmulincstr} are positive matrices (as long as $V$ has at least one nonzero entry per row and column). 
 In practice, one should however be careful because some entries of $W$ and $H$ can numerically be set to zero (because of finite precision). Hence, in our implementation, we use the machine precision as a lower bound for the entries of $W$ and $H$, as recommended in~\cite{gillis2012accelerated}. 
 \end{remark}
 \color{black}

\paragraph{Computational cost}  
The computational cost of Algorithm~\ref{consDisBetaNMF} is asymptotically equivalent to the standard MU for $\beta$-NMF, that is, it requires $\mathcal{O}\left(FNK \right)$ operations per iteration. Indeed, the complexity is mainly driven by matrix products required to compute $C$ and $D$; see~\eqref{eq:updateW}.  
To compute the roots of~\eqref{eq:fmulincstr} corresponding to $H$ using Newton-Raphson, each  iteration requires to compute  $r_j(\mu_j)/r_j'(\mu_j)$ for all $j$ which requires $\mathcal{O}(KN)$ operations (when every entry of $H$ appears in a constraint).  
Finding the roots therefore requires $\mathcal{O}(K N)$ operations times the number of Newton-Raphson iterations. By symmetry, it requires $\mathcal{O}(K F)$ operations to compute the roots corresponding to $W$. 
Because of the quadratic convergence, the number of iterations required for the convergence of the Newton-Raphson method is typically small, namely between 10 to 100 in our experiments using the stopping criterion 
$|r (\mu_j)| \leq 10^{-6}$ for all $j$.  
Therefore, in practice, the overall complexity of Algorithm~\ref{consDisBetaNMF} is dominated by the matrix products that require $\mathcal{O}\left(FNK \right)$ operations. 
The same conclusions apply to the algorithms presented in 
Sections~\ref{sec_ssbetaNMF}, \ref{sec_minvolKLNMF} and~\ref{sec_nonlinear}, and this will be confirmed by our numerical experiments. 

\algsetup{indent=2em}
\begin{algorithm}[ht!]
\caption{$\beta$-NMF with disjoint linear constraints \label{consDisBetaNMF}}
\begin{algorithmic}[1] 
\REQUIRE A matrix $V \in \mathbb{R}^{F \times N}$, 
an initialization $H \in \mathbb{R}^{K \times N}_+$ and $W  \in \mathbb{R}^{F \times K}$, 
a factorization rank $K$, 
a maximum number of iterations, maxiter, 
a value for \jiter{$\beta\not\in(1,2)$}, and 
the linear constraints defined by 
$\mathcal{K}_{i}$, $\alphab_j$ and $b_i$ for $i=1,2,\dots,I$, and 
$\mathcal{B}_{j}$, $\gammab_j$ and $c_j$ for $j=1,2,\dots,J$.  
\ENSURE A rank-$K$ NMF $(W,H)$ of $V$ satisfying constraints in \eqref{eq:betaNMFmodel}. 
    \medskip  
\FOR{$it$ = 1 : maxiter}
    \STATE \emph{\% Update of matrix $H$} 
    \STATE $C \leftarrow W^{T}\left( \left( WH \right)^{.\left( \beta-2 \right) } \odot V \right)$
    \STATE $D \leftarrow W^{T}\left( WH\right)^{.\left( \beta-1 \right) }$
    \FOR{$j = 1 : J$}
            \STATE $\mu_{j} \leftarrow \text{root}\left( r_j(\mu_j) \right)$ \quad \emph{\% see Equation~\eqref{eq:fmulincstr} }
            \STATE $H\left(\mathcal{B}_{j}\right) \leftarrow H\left(\mathcal{B}_{j}\right) \odot \left( \frac{\left[C\left(\mathcal{B}_{j}\right) \right]}{\left[ D\left(\mathcal{B}_{j}\right) -  \mu_{j}  \gamma_j \right]} \right)^{.\left( \eta \left( \beta \right)\right)}$
    \ENDFOR
    \STATE $\mathcal{B}_{c} = \{ (k,n) \ | 1 \leq k  \leq K, 1 \leq n \leq N \} \; \backslash \; \big(\cup_{j}^{J} \mathcal{B}_{j} \big).$ \quad \emph{\% $\mathcal{B}_{c}$ is the complement of $\cup_{j}^{J} \mathcal{B}_{j}$  }
    
    \STATE $H\left(\mathcal{B}_{c}\right) \leftarrow H\left(\mathcal{B}_{c}\right) \odot \left( \frac{\left[C\left(\mathcal{B}_{c}\right) \right]}{\left[ D\left(\mathcal{B}_{c}\right) \right]} \right)^{. \eta \left( \beta \right)}$
	\STATE \emph{\% Update of matrix $W$}
    \STATE $W$ is updated in the same way as $H$, by symmetry of the problem. 
\ENDFOR
\end{algorithmic}  
\end{algorithm}

\section{Showcase 1: Simplex-structured $\beta$-NMF}\label{sec_ssbetaNMF}

In this section, we showcase a particularly important example of $\beta$-NMF with linear disjoint constraints \ji{and no penalization}, namely, the simplex-structured matrix factorization (SSMF) problem. It is defined as follows:  
given a data matrix  $V \in \mathbb{R}^{F \times N}$  and a factorization rank $K$, SSMF  refers  to  the  problem of computing $W$ and $H$ such that $V \approx WH$ and  the  columns  of $H$ lie on the unit simplex, that is, the entries of each column of $H$ are nonnegative and sum to one. SSMF is a powerful tool in many applications such as hyperspectral unmixing in geoscience and remote sensing \cite{6200362,6678258,abdolali2020simplex}, 
document analysis~\cite{chi2012tensors},  
and 
self-modeling curve resolution~\cite{neymeyr2018set}. 
We refer the reader to the recent survey \cite{fu2019nonnegative} for more applications and details about SSMF. 

To understand the underlying significance of SSMF, it is necessary to give more insights on a research topic for which important SSMF techniques were initially developed which is the blind Hyperspectral Unmixing (HU), a main research topic in  remote sensing. The task of blind HU is to decompose a remotely sensed hyperspectral image into endmember spectral signatures and the corresponding abundance maps with limited prior information, usually the only known information being the number of endmembers. In this context, the columns of~$W$ correspond to the endmembers spectral signatures and the columns of~$H$ contain the proportion of the endmembers in each column of~$V$, so the column-stochastic assumption for $H$ naturally holds. The nonnegativity of $W$ follows from the nonnegativity of the spectral signatures.  
We refer to the corresponding problem as simplex-structured nonnegative matrix factorization with the $\beta$-divergence ($\beta$-SSNMF), and is formulated as follows: 
\begin{equation}\label{eq:betaSSNMFmodel}
\underset{W\in \mathbb{R}^{F \times K}_+,H\in \mathbb{R}^{K \times N}_+}{\min}
 D_{\beta}\left(V|WH\right) 
\quad \text{ such that } \quad  
 \eb^{T} \hb_j = 1 \text{ for } 1\leq j \leq N, 
\end{equation}
where $\eb$ is the vector of all ones of appropriate dimension. 
This is particular case of \eqref{eq:betaNMFmodel} where 
\begin{itemize}
    \item the subsets $\mathcal{B}_j$ correspond to the columns of $H$, and there is no subset $\mathcal{K}_i$ (no constraint on $W$),  
    \item $\gammab_j^{T} = \eb$ and $c_j = 1$ for $j=1,2,\dots,N$. 
\end{itemize}
Hence Algorithm~\ref{consDisBetaNMF} can be directly applied to~\eqref{eq:betaSSNMFmodel}.

\paragraph{Numerical experiments}  Let us perform numerical experiments to evaluate the effectiveness of Algorithm~\ref{consDisBetaNMF}  on the simplex-structure $\beta$-NMF problem against existing methods. To the best of our knowledge, the so-called group robust NMF (GR-NMF)  algorithm\footnote{\href{https://www.irit.fr/~Cedric.Fevotte/extras/tip2015/code.zip}{https://www.irit.fr/$\sim$Cedric.Fevotte/extras/tip2015/code.zip}} from \cite{7194802} is the most recent algorithm that is able to tackle problem \eqref{eq:betaSSNMFmodel} for the full range of $\beta$-divergences. 
The approach is not based on Lagrangian multipliers but introduces a change of variables for matrix $H$. This approach, initially used for NMF in \cite{1381036}, does not provide an auxiliary function for the subproblem in $H$ and resort to a heuristic commonly used in NMF, see for example  \cite{4100700,fevotte2009nonnegative}. 
Therefore there is no guarantee that the objective function is decreasing at each update of the abundance matrix, unlike Algorithm~\ref{consDisBetaNMF}.

We apply Algorithm~\ref{consDisBetaNMF} and GR-NMF on three widely used real hyperspectral data sets\footnote{\href{http://lesun.weebly.com/hyperspectral-data-set.html}{http://lesun.weebly.com/hyperspectral-data-set.html}}~\cite{zhu2017hyperspectral}:  
\begin{itemize}
    
    \item Samson:   156 spectral bands with 95$\times$95 pixels, containing mostly 3 materials $(K=3)$, 
    namely ``Soil", ``Tree" and ``Water". 

\item Jasper Ridge: 198 spectral bands with 100$\times$100 pixels, containing mostly 4 materials $(K=4)$, 
    namely ``Road", ``Soil", ``Water" and ``Tree".

    \item Cuprite: 188 spectral bands with 250$\times$190 pixels, containing mostly 12 types of minerals $(K=12)$.

\end{itemize}

$\beta$-SSNMF has shown itself as a powerful one to tackle blind HU, hence this comparative study between Algorithm~\ref{consDisBetaNMF} and GR-NMF~\cite{7194802} focuses on the convergence aspects including the  evolution of the objective function and the runtime. The algorithms are compared\revise{\footnote{\revise{
For $\beta = 3/2$, we had an error in our derivations, and use~\eqref{eq:updateW} with $\eta(\beta) = 1$ for $\beta \in (1,2)$; see the discussion after~\eqref{eq:updateW}. 
However, the corresponding MU always decreases the objective function values (which we were monitoring), although we do not have a theoretical justification for this. A possible approach to obtain such as result would be to come up with a majorizer of the majorizer that has a closed-form minimizer given by~\eqref{eq:updateW} with $\eta(\beta) = 1$ for $\beta \in (1,2)$. 
}
}} 
for $\beta \in  \left\{0,\frac12,1,\frac32,2\right\}$. 
To report the results, we use the relative objective function, denoted $\Bar{F}(W,H)$ and defined as\footnote{For the Frobenius norm, that is, $\beta = 2$, the relative error is typically defined as $\frac{D_\beta(V|WH)}{D_\beta(V|0)}$ meaning that the trivial solution used is the all-zero matrix. However, for other $\beta$-divergences, the value of $D_\beta(V|0)$ might not be defined; in particular, for $\beta \leq 1$ and $v_{fn} > 0$ for some $f,n$. } 
\begin{equation*}
    \Bar{F}(W,H) = \frac{D_\beta(V|WH)}{D_\beta(V|v \eb\eb^T)} , 
\end{equation*} 
where $v = \frac{\eb^T V \eb}{FN}$ is the average of the entries of $V$. 
The relative error  $\Bar{F}$ should be between 0 and 1: it is equal to 0 for an exact decomposition with $V = WH$, and is equal to 1 for a trivial rank-one approximation where all entries are equal to the average of the entries of $V$. 
This allows to meaningfully interpret the results, especially since we consider in this comparative study multiple values for $\beta$. In fact, the degree of homogeneity of the $\beta$-divergence is a  function of $\beta$.
For example, if all the entries of the input matrix are multiplied by 10 and keeping the same  NMF solution properly scaled,   the squared Frobenius error ($\beta$ = 2) is multiplied by 100 while the IS-divergence ($\beta$ = 0) is not affected. 
\color{black}

As for all tests performed in this paper, the algorithms are tested on a desktop computer with Intel Core i7-8700@3.2GHz CPU and 32GB memory. The codes are written in MATLAB R2018a, and available from \href{https://sites.google.com/site/nicolasgillis/}{https://sites.google.com/site/nicolasgillis/}.  
For all simulations, the algorithms are run for 20 random initializations of $W$ and $H$ (each entry sampled from the uniform distribution in $[0,1]$). 
Table~\ref{table:ssbetaNMFperf} reports the average and standard deviation of the runtime (in seconds) as the final value for the {relative} objective function over these 20 runs for a maximum of 300 iterations. 
\begin{table}[ht!]
\centering
\caption{Runtime performance in seconds and final value of {relative} objective function ${\Bar{F}}_\text{end}(W,H)$ for Algorithm~\ref{consDisBetaNMF} and the GR-NMFreported for $\beta \in  \left\{0,\frac12,1,\frac32,2\right\}$. The table reports the average and standard deviation over 20 random initializations with a maximum of 300 iterations for three hyperspectral data sets. A bold entry indicates the best value for each experiment. 
}
\label{table:ssbetaNMFperf}
\setlength{\tabcolsep}{2pt}
\resizebox{\columnwidth}{!}
{
\begin{tabular}{|c|c|c|c|c|c|c|}
  \hline
  Algorithms       & \multicolumn{2}{|c|}{Samson} & \multicolumn{2}{|c|}{Jasper Ridge} & \multicolumn{2}{|c|}{Cuprite}\\
                   & runtime (s.) & ${\Bar{F}}_\text{end}(W,H)$ & runtime (s.) & ${\Bar{F}}_\text{end}(W,H)$ & runtime (s.) & ${\Bar{F}}_\text{end}(W,H)$ \\
  \hline
                   &\multicolumn{6}{|c|}{$\beta=2$}\\
  \hline 
   Algorithm~\ref{consDisBetaNMF}   &\textbf{16.62$\pm$0.15}  &{\textbf{(1.89$\pm$0.04)}$10^{-3}$}  &\textbf{22.86$\pm$0.08}  &{\textbf{(4.68 $\pm$ 0.39)$10^{-3}$}} &121.04 $\pm$ 0.62 &{ \textbf{(0.98 $\pm$ 0.06)$10^{-3}$}} \\

  GR-NMF &18.23$\pm$0.29  &{(1.91$\pm$0.05)$10^{-3}$}  &25.32$\pm$0.16  &{(5.87 $\pm$ 1.22)$10^{-3}$} &\textbf{114.27 $\pm$ 0.20} &{ (1.29 $\pm$ 0.07)$10^{-3}$}\\

  \hline
                  &\multicolumn{6}{|c|}{$\beta=3/2$}\\
  \hline 
   Algorithm~\ref{consDisBetaNMF}   &\textbf{63.69$\pm$0.40}  &{\textbf{(2.52 $\pm$0.78)$10^{-3}$}}  & \textbf{89.23 $\pm$0.30} &{ \textbf{(4.92 $\pm$ 0.29)$10^{-3}$}} &\textbf{421.49 $\pm$ 2.79} &{\textbf{(1.54 $\pm$ 0.07)$10^{-3}$}}\\

  GR-NMF &80.09$\pm$0.60    &{(2.60 $\pm$0.63)$10^{-3}$}    &112.72 $\pm$0.67  &{(6.32 $\pm$ 1.37)$10^{-3}$}  &508.57 $\pm$ 3.50 &{(2.01 $\pm$ 0.09)$10^{-3}$}\\

  \hline
                  &\multicolumn{6}{|c|}{$\beta=1$}\\
  \hline 
   Algorithm~\ref{consDisBetaNMF}   &\textbf{18.33 $\pm$ 0.08} &{\textbf{(3.54 $\pm$ 0.27)}$10^{-3}$ } &\textbf{24.82 $\pm$ 0.35}  &{ \textbf{(6.07 $\pm$ 0.21)$10^{-3}$}} &\textbf{182.98 $\pm$ 14.14} &{\textbf{(2.07 $\pm$ 0.09)$10^{-3}$}}\\

  GR-NMF &44.78 $\pm$ 0.18 &{(3.77 $\pm$ 0.38 ) $10^{-3}$} &62.83 $\pm$ 0.76  &{ (7.26 $\pm$ 1.50)$10^{-3}$ }&370.25 $\pm$ 21.33 &{(2.67 $\pm$ 0.10)$10^{-3}$}\\

  \hline
  
                  &\multicolumn{6}{|c|}{$\beta=1/2$}\\
  \hline 
   Algorithm~\ref{consDisBetaNMF}   &\textbf{89.80 $\pm$ 0.65}  &{(7.21 $\pm$ 0.75)$10^{-3}$}  &\textbf{126.43 $\pm$ 0.61}  &{ \textbf{(1.08 $\pm$ 0.10)$10^{-2}$}}&682.80 $\pm$ 3.32 &{\textbf{(3.13 $\pm$ 0.15)$10^{-3}$}}\\

  GR-NMF &102.21 $\pm$ 0.72  &{\textbf{(6.93 $\pm$ 0.88)$10^{-3}$}}  &141.75 $\pm$ 0.69  &{ (1.12 $\pm$ 0.13)$10^{-2}$} &\textbf{642.49 $\pm$ 1.22} &{(3.14 $\pm$ 0.14)$10^{-3}$}\\
  \hline
                  &\multicolumn{6}{|c|}{$\beta=0$}\\
  \hline 
   Algorithm~\ref{consDisBetaNMF}   &\textbf{52.89$\pm$0.54}  &{(4.60 $\pm$ 0.66)$10^{-2}$}  &\textbf{69.59$\pm$0.44}  &{\textbf{(3.76 $\pm$ 0.11)$10^{-2}$}} &479.84 $\pm$ 16.02 &{(4.39 $\pm$ 0.31)$10^{-3}$}\\

  GR-NMF &55.61$\pm$0.47  &{\textbf{(4.22 $\pm$ 0.79)$10^{-2}$  }}&77.87$\pm$0.63  &{(3.76 $\pm$ 0.44)$10^{-2}$} &\textbf{354.65 $\pm$6.01} &{\textbf{(3.35 $\pm$ 0.10)$10^{-3}$}}\\

  \hline
\end{tabular} 
}
\end{table}

We observe that  Algorithm~\ref{consDisBetaNMF} outperforms the GR-NMF in terms of runtime and final values for the {relative} objective function for all test cases except when $\beta=0$ for the Samson and Cuprite data sets. In particular, for $\beta=1$,  Algorithm~\ref{consDisBetaNMF} is up to 2.5 times faster than the GR-NMF.
For the Cuprite data set with $\beta=1/2$, Algorithm~\ref{consDisBetaNMF} and GR-NMF perform similarly. We also observe that the standard deviations obtained with Algorithm~\ref{consDisBetaNMF} are in general significantly smaller for all $\beta$, except for $\beta=0$ for the Samson and Cuprite data sets. 

In the supplementary {material} S1, we provide figures that show the  evolution of the {relative} objective function values with respect to iterations, and that confirm the observations above. 
\section{Showcase 2: minimum-volume KL-NMF}  \label{sec_minvolKLNMF}


In this section, we showcase another important example of $\beta$-NMF with linear disjoint constraints, namely, the minimum volume NMF with the $\beta$-divergences (min-vol $\beta$-NMF) model. This model is based on the minimization of $\beta$-divergences including a
penalty term promoting solutions with minimum volume spanned by the columns of the matrix $W$. It is defined as follows:
\begin{equation}\label{eq:4}
 \underset{W\in \mathbb{R}^{F \times K}_+,H\in \mathbb{R}^{K \times N}_+}{\min}  
 D_{\beta}(V|WH) + \lambda \text{vol}(W) \quad \text{ such that } \quad 
 W^T e = e. 
\end{equation} 
where
$\lambda$ is a penalty parameter,
and 
$\text{vol}(W)$ is a function measuring the volume spanned by the columns of $W$.  In \cite{9084229}, the authors use 
$\text{vol}(W)=\logdet(W^{T}W+\delta I)$, 
 where $\delta$ is a small positive constant that prevents $\logdet(W^{T}W)$ to go to $-\infty$ when $W$ tends to a rank-deficient matrix (that is, when $\text{rank}(W)< K$). 
This model is particularly powerful as it leads to identifiability which is crucial in many applications such as in hyperspectral imaging or audio source separation~\cite{fu2019nonnegative}. 
Indeed, under some mild assumptions and in the exact case, authors prove in \cite{9084229} that \eqref{eq:4} is able to identify the groundtruth factors $(W^\#,H^\#)$ that generated the input data $V$, in the absence of noise. 
In \cite{9084229}, \eqref{eq:4} is used for blind audio source separation. In a nutshell, blind audio source separation consists in isolating and extracting unknown sources based on an observation of their mix recorded with a single microphone\footnote{We invite the interested reader to watch the video \href{https://www.youtube.com/watch?v=1BrpxvpghKQ}{https://www.youtube.com/watch?v=1BrpxvpghKQ} to see the application of min-vol KL-NMF on the decomposition of a famous song from the city of Mons.}. 
We have to mention that model \eqref{eq:4} is also well suited for hyperspectral imaging as discussed in \cite{gillis2020bk}.

In the next subsections, we show that we can tackle the min-vol $\beta$-NMF optimization problem defined in \eqref{eq:4} with the general framework presented in Section~\ref{sec_generalframework} in the case $\beta = 1$.  


\subsection{Problem formulation and algorithm}

As the minimum-volume penalty of model \eqref{eq:4} concerns matrix $W$ only, the main challenge concerns the update of $W$. Indeed, the update of $H$ is simply the one from~\cite{algoNMFlee}. Let us therefore consider the subproblem in $W$ for $H$ fixed: 
\begin{equation}\label{eq:betaminvolNMFinW} 
\underset{W\in \mathbb{R}^{F \times K}_+}{\min}
D_{\beta}\left(V|WH\right)+ \lambda \logdet(W^{T}W+\delta I) 
\quad 
\text{such that} \quad  \eb^{T} \wb_i = 1 \text{ for } 1\leq i \leq K. 
\end{equation}
 Compared to the general model \eqref{eq:betaNMFmodel}, we have that 
\begin{itemize}
    \item the subsets $\mathcal{K}_{i}$ correspond to the columns of $W$, and there is no subset $\mathcal{B}_j$, 
    \item $\alphab_i^{T} = \eb$ and $b_i = 1$ for $1\leq i \leq K$. 
\end{itemize}
To upper bound $\logdet(W^{T}W+\delta I)$ as required by~\eqref{eq:Lipschtiz_case} in Assumption~\ref{ass:phi}, 
we majorize it using a convex quadratic separable auxiliary function  provided in \cite[Eq.~(3.6)]{9084229} and which is derived as follows. First, the concave function $\logdet(Q)$ for $Q \succ 0$ can be upper bounded using the first-order Taylor approximation: for any $\widetilde{Q} \succ 0$, 
\[
\logdet( Q ) 
\; \leq \;  \logdet( \widetilde{Q} ) + \langle 
\widetilde{Q}^{-1} , Q - \widetilde{Q} \rangle 
\; =  \;
\langle \widetilde{Q}^{-1} , Q \rangle + \text{cst}, 
\]
where cst is some constant independent of $Q$.  
For any $W, \widetilde{W}$, and denoting  
$\widetilde{Q} = \widetilde{W}^T \widetilde{W} + \delta I \succ 0$, we obtain  
\[
\logdet( W^T W + \delta I ) 
\,  \leq \, 
\left\langle \widetilde{Q}^{-1} ,  
W^T W  \right\rangle + \text{cst} 
\, = \, \tr ( W \widetilde{Q}^{-1} W^T ) + \text{cst} , 
\]
which is a convex quadratic and Lipschitz-smooth function in $W$. 
In fact, letting $\widetilde{Q}^{-1} = DD^T$ be a decomposition (such as Cholesky) of  $\widetilde{Q}^{-1} \succ 0$, we have 
$\tr ( W \widetilde{Q}^{-1} W ) = \| WD \|_F^2$, 
from which~\eqref{eq:Lipschtiz_case} can be derived easily; see  \cite{9084229} for the details. 
With this and following our framework from Section~\ref{sec_generalframework}, we obtain the Lagrangian function 
\begin{equation}\label{eq:Gmuminvol} 
G^{\mu}\big(W|\widetilde{W}\big) = \sum_{f} G\left( \wb_f |\widetilde{\wb}_f \right) + \lambda \left( \sum_{f} \bar{l}\left( \wb_f |\widetilde{\wb}_f\right) + c\right) 
 +\mub^{T} \sum_{f}\left(\wb_{f} - \frac1F \eb \right), 
\end{equation} 
where \ji{$\wb_f\in$ denotes the $f$-th row of $W$}, 
$G$ is given by \eqref{eq:G}, $\bar{l}$ by~\cite[Eq.~(3.6)]{9084229} and derived as explained above, and $c$ is a constant. Let $\mub$ is the vector Lagrange multipliers of dimension $K$ 
associated to each linear constraint $\eb^{T} \wb_i = 1$. 
Exactly as before (hence we omit the details here), $G^{\mu}$ is separable and,  given $\mub$, one can compute the closed-form solution: 
\begin{equation}\label{eq:updateWminvolNMF}
\begin{aligned}
& W^\star(\mu) = \widetilde{W} \odot  
\frac{\left[\left[ \left[ C + \eb \mub^{T} \right]^{.2}+S \right]^{.\frac{1}{2}} - \left(C + \eb  \mub^{T}\right)\right]}{\left[ D  \right] }, 
\end{aligned}
\end{equation} 
where 
$C=e_{F,N}  H^{T} - 4\lambda \big( \widetilde{W} Y^{-}\big)$, $D =4\lambda \widetilde{W} \left( Y^{+}+Y^{-}\right)$, 
and 
$S=8 \lambda \widetilde{W} \left( Y^{+}+Y^{-}\right)\odot \left( \frac{\left[ V \right]}{\left[\widetilde{W}H\right]} H^{T} \right)$ 
with 
$Y=Y^{+}-Y^{-}=\big(\widetilde{W}^{T}\widetilde{W}+\delta I\big)^{-1}$, 
$Y^{+} = \text{max}(Y,0) \geq 0$ 
and 
$Y^{-} = \text{max}(-Y,0) \geq 0$, and 
$e_{F,N}$ is the $F$-by-$N$ matrix of all ones. 
As proved in Proposition~\ref{prop:mu_unicity}, the constraint  
$W^\star(\mu)^T \eb = \eb$ is satisfied for a unique $\mub$ in $\left(-\infty,t \right)$ where $t=\infty$ in this case. We can therefore use a Newton-Raphson method to find the $\mu_{i}$ with quadratic rate of convergence, see Proposition~\ref{NR}.
Algorithm~\ref{dcminvolKLNMF} summarizes our method to tackle~\eqref{eq:4}. 
\algsetup{indent=2em}
\begin{algorithm}[ht!]
\caption{Min-vol KL-NMF \label{dcminvolKLNMF}}
\begin{algorithmic}[1] 
\REQUIRE A matrix $V \in \mathbb{R}^{F \times N}$, an initialization $H \in \mathbb{R}^{K \times N}_+$, an initialization $W  \in \mathbb{R}^{F \times K}$ , a factorization rank $K$, and a maximum number of iterations, maxiter, the parameters $\delta > 0$ and $\lambda > 0$. 
\ENSURE A min-vol rank-$K$ NMF $(W,H)$ of $V$ satisfying constraints in \eqref{eq:4}. 
    \medskip  
\FOR{$it$ = 1 : maxiter}
    \STATE \emph{\% Update of matrix $H$}
	\STATE  $H \leftarrow H \odot \frac{\left[  W^{T} \left( \frac{\left[ V\right]}{\left[ WH\right]} \right) \right]}{\left[  W^{T} e_{F,N} \right]}$ \quad 
	\STATE \emph{\% Update of matrix $W$}
	\STATE $Y \leftarrow \left(W^{T}W+\delta I \right)^{-1}$
	\STATE $Y^{+} \leftarrow \text{max}\left(Y,0\right)$
	\STATE $Y^{-} \leftarrow \text{max}\left(-Y,0\right)$
    \STATE $C \leftarrow e_{F,N}  H^{T} - 4\lambda \left( W Y^{-}\right)$
    \STATE $S \leftarrow 8 \lambda W \left( Y^{+}+Y^{-}\right)\odot \left( \frac{\left[ V \right]}{\left[WH\right]} H^{T} \right)$
    \STATE $D  \leftarrow 4\lambda W \left( Y^{+}+Y^{-}\right)$
            \STATE $\mub \leftarrow 
            \text{root}\left( W^\star(\mu)^T \eb = \eb  \right)$ over $\mathbb{R}^K$  \quad \emph{\% see \eqref{eq:updateWminvolNMF} for the expression of $W^\star(\mu)$}   
    \STATE $W \leftarrow W \odot  \frac{\left[\left[ \left[ C + e  \mub^{T} \right]^{.2}+S \right]^{.\frac{1}{2}} - \left(C + \eb  \mub^{T}\right)\right]}{\left[ D  \right] }$
\ENDFOR
\end{algorithmic}  
\end{algorithm} 


\subsection{Numerical experiments}

In this section we compare baseline KL-NMF (that is, the standard MU), the min-vol KL-NMF from \cite[Algorithm 1]{9084229} that 
solves~\eqref{eq:4} using MU combined with line search (min-vol KL-NMF LS), 
and Algorithm~\ref{dcminvolKLNMF} applied to the spectrogram of two monophonic piano sequences considered in \cite{9084229}. The first audio sample is the first measure of ``Mary had a little lamb", a popular English song. The second audio sample corresponds to the first 30 seconds of ``Prelude and Fugue No.1 in C major" from de Jean-Sebastien Bach played by Glenn Gould\footnote{\href{https://www.youtube.com/watch?v=ZlbK5r5mBH4}{https://www.youtube.com/watch?v=ZlbK5r5mBH4}}. 
We use the following three setups: 
\begin{itemize}
    \item Setup $\sharp$1: sample ``Mary  had  a  little  lamb" with \mbox{$K=3$}, 200 iterations. 
    \item Setup $\sharp$2: sample ``Mary  had  a  little  lamb" with \mbox{$K=7$}, 200 iterations. 
    \item Setup $\sharp$3: ``Prelude and Fugue No.1 in C major" with $K=16$, 300 iterations.
\end{itemize}
For each setup, the algorithms are run for the same 20 random initializations of $W$ and $H$. 
Table~\ref{table:runtimeperf} reports the average and standard deviation of the runtime (in seconds) over these 20 runs. Table~\ref{table:lossfunperf} reports the average and standard deviation of the final values for $\beta$-divergences (data fitting term) and the objective function of \eqref{eq:4} over these 20 runs for min-vol KL-NMF LS and Algorithm~\ref{dcminvolKLNMF}. For this last comparison, the value for the penalty weight $\lambda$ has been chosen so that KL-NMF leads to reasonable solutions for $W$ and $H$. More precisely, the values for $\lambda$ are chosen so that the initial value of  $\frac{\lambda \left| \logdet( {W^{(0)}}^T W^{(0)} + \delta I)\right|}{D_{\beta}\left(V|WH \right)}$  is  equal to 0.1, 0.1 and 0.022 for setup $\sharp$1, setup $\sharp$2 and setup $\sharp$, respectively. 
\begin{center}
\begin{table}[ht!]
\begin{center}
\caption{ Runtime  performance in seconds of baseline KL-NMF, min-vol KL-NMF LS and Algorithm~\ref{dcminvolKLNMF}. The table reports the average and standard deviation over 20 random initializations. 
} 
\label{table:runtimeperf}
\begin{tabular}{|c|c|c|c|}
  \hline
  Algorithms       & \multicolumn{3}{|c|}{runtime in seconds} \\
                   & setup $\sharp$1 & setup $\sharp$2 & setup $\sharp$3  \\
  \hline
  baseline KL-NMF   & 0.53$\pm$0.03 & 0.45$\pm$0.02 & 4.32$\pm$0.30  \\
 
  \hline
  min-vol KL-NMF LS \cite{9084229} & 3.79$\pm$0.13 & 2.39$\pm$0.30 & 10.19$\pm$1.28     
      \\  \hline
  Algorithm~\ref{dcminvolKLNMF}    & 0.58$\pm$0.03 & 0.66$\pm$0.03 & 4.80$\pm$ 0.38                  \\  \hline
\end{tabular} 
\end{center}
\end{table}
\end{center}



\begin{center}
\begin{table}[ht!]

\begin{center}
\caption{
Final values for $D_{\beta}$ and the penalized objective $\Psi$ from \eqref{eq:4} obtained with min-vol KL-NMF LS and Algorithm~\ref{dcminvolKLNMF}. The table reports the average and standard deviation over 20 random initializations for three experimental setups. A bold entry indicates the best value for each experiment. 
} 
\label{table:lossfunperf}
\begin{tabular}{|c|c|c|c|}
\hline
\multicolumn{2}{|c|}{\hfill }   & min-vol KL-NMF LS \cite{9084229}& Algorithm~\ref{dcminvolKLNMF}  \\   \hline 
 \multirow{2}{4em}{setup $\sharp$1} & $D_{\beta,\text{end}}$ & (3.52 $\pm$ 0.03)$10^{3}$ & \textbf{(2.31 $\pm$ 0.01)$10^{3}$} \\
                 & $\Psi_\text{end}$              & (4.17 $\pm$ 0.03)$10^{3}$ & \textbf{(3.08 $\pm$ 0.01)$10^{3}$} \\ \hline
 \multirow{2}{4em}{setup $\sharp$2} & $D_{\beta,\text{end}}$ & (3.54 $\pm$ 0.03)$10^{3}$ & \textbf{(1.77 $\pm$ 0.02)$10^{3}$} \\
                 & $\Psi_\text{end}$              & (4.42 $\pm$ 0.04)$10^{3}$ & \textbf{(2.87 $\pm$ 0.02)$10^{3}$} \\ \hline
 \multirow{2}{4em}{setup $\sharp$3} & $D_{\beta,\text{end}}$ & (7.77 $\pm$ 0.23)$10^{3}$ & \textbf{(4.67 $\pm$ 0.08)$10^{3}$} \\
                 & $\Psi_\text{end}$              & (9.14 $\pm$ 0.20)$10^{3}$ & \textbf{(6.50 $\pm$ 0.06)$10^{3}$} \\ \hline

\end{tabular}
\end{center}
\end{table}
\end{center}

We observe that the runtime of Algorithm~\ref{dcminvolKLNMF} is close to the baseline KL-NMF algorithm which confirms the negligible cost of the Newton-Raphson steps to compute $\mub^\star$ as discussed in Section~\ref{betanmf_disjointconstraint}.
On the other hand, since no line search is needed, we have a drastic acceleration from 2x to 7x compared to the backtracking line-search procedure integrated in min-vol KL-NMF LS \cite{9084229}. 
Moreover, we observe in Table~\ref{table:lossfunperf} that Algorithm~\ref{dcminvolKLNMF} outperforms min-vol KL-NMF LS in terms of final values for the data fitting term and objective function values, with lower standard deviations. 
\section{Extension to quadratic disjoints constraints}\label{sec_nonlinear}
Our general framework presented in Section~\ref{sec_generalframework} applies to $\beta$-NMF under disjoint linear equality constraints with penalty terms satisfying Assumption~\ref{ass:phi}; 
see problem~\eqref{eq:betaNMFmodel}. 
We have showcased our approach on $\beta$-SSNMF in Section~\ref{sec_ssbetaNMF} and on min-vol KL-NMF under sum-to-one constraints on the columns of $W$ in Section~\ref{sec_minvolKLNMF}. 
In this section, we show that the same framework can be extended to other simple constraints, namely disjoint quadratic constraints. 

We consider sparse $\beta$-NMF for $\beta=1$ where 
the rows of $H$ are penalized with the $\ell_1$ norm and 
each column of $W$ have a fixed $\ell_2$ norm. 
We show that MU satisfying the set of constraints can be derived which we apply on blind HU. 

\subsection{Problem formulation and algorithm}

In this section we consider the following model involving quadratic disjoints constraints, that we refer to as hyperspheric-structured sparse $\beta$-NMF:
\begin{equation}\label{eq:betahsminvolNMFmodel}
 \underset{W\in \mathbb{R}^{F \times K}_+,H\in \mathbb{R}^{K \times N}_+}{\min} 
 D_{\beta}(V|WH) + \sum_{k=1}^{K}\lambda_{k} \left\| H(k,:) \right\|_{1} \\
\quad \text{such that} \quad  
 \eb^{T} \wb_j^{.(2)} = \rho \text{ for } 1\leq i \leq K, 
\end{equation}
where 
$\lambda_{k}$ is a penalty weight to control the sparsity of the $k$-th row of $H$, and the quadratic constraints require the columns of $W$ to lie on the surface of a hyper-sphere centered at the origin with radius $\sqrt{\rho} > 0$. 
Without this normalization, the $\ell_{1}$-norm regularization would make $H$ tends to zero and $W$ grows to infinity. 

As done before, we update $W$ and $H$ alternatively. We tackle the subproblem in $H$ with $W$ fixed based on the MU developed in \cite{fevotte2009nonnegative} and guaranteed to decrease the objective function: 
\begin{equation}\label{eq:updateHhssparseNMF}
\begin{aligned}
& H = \widetilde{H} \odot \frac{\left[  W^{T}\left(  V \odot \big[ W\widetilde{H}\big]^{.(\beta-2)} \right) \right]}{\left[  W^{T} \big[ W\widetilde{H}\big]^{.(\beta-1)} + \lambdab e^T \right]} , 
\end{aligned}
\end{equation} 
where $\lambdab\in\mathbb{R}_+^K$ is the vector of penalty weights.
It remains to compute an update for $W$. To do so, we use the convex separable auxiliary function $G$ from \cite{Fevotte_betadiv} constructed at the current iterate $\widetilde{W}$, from which we obtain, as before, the Lagrangian function 
\begin{equation}\label{eq:Gmuhsminvol}
\begin{aligned}
G^{\mu}\big(W|\widetilde{W}\big)&= \sum_{f} G\left( w_f |\widetilde{w}_f \right) + \sum_{k}\lambda_{k} \left\| H(k,:) \right\|_{1}+ {\mub}^{T} \sum_{f}\left(w_{f}^{.(2)} - \frac1F\rho \eb \right), 
\end{aligned}
\end{equation} 
where $\mub\in\mathbb{R}^K$ is the vector of Lagrange multipliers associated to the constraint $\eb^{T} \Wb^{.(2)} = \rho \eb^{T}$. 
Exactly as before (hence we omit the details here),  given $\mub$, 
one can obtain a closed-form solution: 
\begin{equation}\label{eq:updateHhsminvolNMF}
\begin{aligned}
& W^\star(\mub) = \frac{\left[\left[ \left[ C \right]^{.2}+8 \left( 
\eb \mub^{T}\right) \odot S \right]^{.\frac{1}{2}} - C\right]}{\left[ 4  \eb \mub^{T} \right] } , 
\end{aligned}
\end{equation}   
where  $C=e_{F,N}  H^{T}$ and $S=\widetilde{W}\odot \left(  \frac{\left[ V \right]}{\left[ \widetilde{W} {H}\right]} H^{T} \right)$. 
Let us now write the expression of the quadratic constraint $\sum_{f}(W^{\star}(\mub)_{f,i})^{2}-\rho=0$ for one specific column of $W$, say the $i$-th: 
\begin{equation} \label{eq:rootsparseNMF}
r_{i}\left( \mu_i \right) := \sum_{f}(W^{\star}_{f,i}(\mu_i))^{2}-\rho 
= \sum_{f} \left(\frac{\sqrt{ \left( C_{f,i}  \right)^{2}+8 \mu_{i} S_{f,i} } - C_{f,i} }{ 4 \mu_{i} }\right)^{2} -\rho =0. 
\end{equation} 
Computing the Lagrangian multiplier $\mu_{i}$ to satisfy the constraint requires computing the roots of the functions $r_{i}\left( \mu_i \right)$. We can show that each $W^{\star}_{f,i}(\mu_i)$ \eqref{eq:updateHhsminvolNMF} is a monotone decreasing, nonnegative convex function over $(0,+\infty)$. Therefore $\sum_{f}(W^{\star}_{f,i}(\mu_i))^{2}$ is also monotone decreasing and convex in $\mu_{i}$ over $(0,+\infty)$. Indeed, let $g$ : $\mathbb{R}_{+} \rightarrow \mathbb{R}_{+}$ be a monotone decreasing, nonnegative convex function. If $g$ is twice-differentiable, then $\left( g^{2}\right)'' = 2\left( g'\right)^{2}+2gg'' \geq 0$ since $g,g''\geq 0$ and $\left( g^{2} \right)'=2g'g \leq 0$ since  $g\geq 0$, $g'\leq 0$ by hypothesis. Now we can conclude that $r_{i}\left( \mu_i \right)$ is a monotone decreasing convex function over $(0,+\infty)$. Moreover, using Hospital's rule, we have: 
\begin{equation*}
\underset{\mu_{i}\rightarrow 0^{+}}{\text{lim}} \sum_{f}(W^{\star}_{f,i}(\mu_i))^{2}-\rho= +\infty 
\quad  \text { and } \quad 
\underset{\mu_{i}\rightarrow +\infty}{\text{lim}} \sum_{f}(W^{\star}_{f,i}(\mu_i))^{2}-\rho = -\rho <0,
\end{equation*} 
since $\rho > 0$. Therefore, the root of $r_{i}\left( \mu_i \right)$ is unique over  $(0,+\infty)$. We use a Newton-Raphson method to solve the problem.
Algorithm~\ref{dchsminvolKLNMF} summarizes our method. 

\algsetup{indent=2em}
\begin{algorithm}[ht!]
\caption{Hyperspheric-structured sparse KL-NMF \label{dchsminvolKLNMF}}
\begin{algorithmic}[1] 
\REQUIRE A matrix $V \in \mathbb{R}^{F \times N}$, an initialization $H \in \mathbb{R}^{K \times N}_+$, an initialization $W  \in \mathbb{R}^{F \times K}$ , a factorization rank $K$, 
a maximum number of iterations, maxiter, a weight vector $\lambda >0$.
\ENSURE A sparse rank-$K$ NMF $(W,H)$ of $V$ satisfying constraints in \eqref{eq:betahsminvolNMFmodel}. 
    \medskip  
\FOR{$it$ = 1 : maxiter}
	\STATE \emph{\% Update of matrix $H$}
    \STATE $H \leftarrow H \odot \frac{\left[  W^{T}\left(  V \odot \big[ WH\big]^{.(\beta-2)} \right) \right]}{\left[  W^{T} \big[ WH\big]^{.(\beta-1)} + \lambdab e^T \right]}$
    \STATE \emph{\% Update of matrix $W$}
    \STATE $C \leftarrow  e_{F,N} H^T$
    \STATE $S \leftarrow W \odot \left( \frac{\left[ V \right]}{\left[WH\right]} H^T \right)$
    \FOR{$j$ = 1 : K}
            \STATE $\mu_{i} \leftarrow \text{root}\left( r_{i}\left( \mu_i \right) \right)$ over $(0,+\infty)$ \emph{\% See Equation~\eqref{eq:rootsparseNMF}}
    \ENDFOR
	\STATE  $W \leftarrow \frac{\left[\left[ \left[ C \right]^{.2}+8 \left(  \eb \mub^T \right) \odot S \right]^{.\frac{1}{2}} - C\right]}{\left[ 4 \mub \eb^{T} \right] }$  
\ENDFOR
\end{algorithmic}  
\end{algorithm} 


\subsection{Numerical experiments}

In this section, we perform numerical experiments to evaluate the effectiveness of Algorithm~\ref{dchsminvolKLNMF} on the HU problem. 
To the best of our knowledge, sparse $\beta$-NMF\footnote{\href{http://www.jonathanleroux.org/software/sparseNMF.zip}{http://www.jonathanleroux.org/software/sparseNMF.zip}}  from \cite{Leroux} is the most recent algorithm that is able to tackle problem \eqref{eq:betahsminvolNMFmodel} for the KL-divergence by integrating the $\ell_2$-normalization for each update of matrix $W$. 
This approach is similar to that of~\cite{7194802} for $\beta$-SSNMF, that is, it uses  parametrization, and resort to a heuristic with no guarantee on the decrease of the objective function. 
We refer to this algorithm as $\beta$-SNMF. 

We apply Algorithm~\ref{dchsminvolKLNMF} and $\beta$-SNMF~\cite{Leroux} to the three real hyperspectral datasets detailed in Section~\ref{sec_ssbetaNMF}. This comparative study focuses on the convergence aspects including the evolution of the objective function and the runtime; we refer the interested reader to the Supplementary Material S2 for qualitative result on the ability of sparse $\beta$-NMF to decompose such images. 
For all simulations, the algorithms are ran for 20 random initializations of $W$ and $H$, the entries of the penalty weight $\lambdab$ has been set to 0.1, 0.05 and 0.05 for Samson, Jasper Ridge and Cuprite data sets, respectively. In order to fairly compare both algorithms, $\rho$ has been set to 1 as $\beta$-SNMF considers a $\ell_2$-normalization for the columns of $W$, and the entries of the weight vector $\lambdab$ in Algorithm~\ref{dchsminvolKLNMF} have the same values as $\beta$-SNMF requires to use the same values for all rows of $H$. Table~\ref{table:hsbetaNMFperf} reports the average and standard deviation of the runtime (in seconds) as the final value for the objective function over these 20 runs for a maximum of 300 iterations. 
Figure~\ref{fig:hsbetaNMF_lossfun_compa} displays the objective function values. 
\begin{table}[ht!]
\centering
\caption{Runtime performance in seconds and final value of objective function $\Phi_\text{end}(W,H)$ for Algorithm~\ref{dchsminvolKLNMF} and $\beta$-SNMF. The table reports the average and standard deviation over 20 random initializations with a maximum of 300 iterations for three hyperspectral data sets. A bold entry indicates the best value for each experiment.  
} 
\label{table:hsbetaNMFperf}
\setlength{\tabcolsep}{2pt}
\resizebox{\columnwidth}{!}{%
\begin{tabular}{|c|c|c|c|c|c|c|}
  \hline
  Algorithms       & \multicolumn{2}{|c|}{Samson data set} & \multicolumn{2}{|c|}{Jasper Ridge data set} & \multicolumn{2}{|c|}{Cuprite data set}\\
                   & runtime (sec) & $\Phi_\text{end}(W,H)$ & runtime (sec) & $\Phi_\text{end}(W,H)$ & runtime (sec) & $\Phi_\text{end}(W,H)$ \\
  \hline 
   Algorithm~\ref{dchsminvolKLNMF}   &11.07$\pm$0.19  &(2.68$\pm$0.00)$10^{3}$  &15.67$\pm$0.17  &\textbf{(4.65 $\pm$ 0.00)$10^{3}$} &70.16 $\pm$ 0.85 & \textbf{(2.12 $\pm$ 0.00)$10^{3}$} \\ 
  $\beta$-SNMF~\cite{Leroux}  &\textbf{7.63$\pm$0.13}  &(2.68$\pm$0.00)$10^{3}$ &\textbf{10.98$\pm$0.18}  &(4.71 $\pm$ 0.00)$10^{3}$ &\textbf{51.86 $\pm$ 0.74} & (2.18 $\pm$ 0.00)$10^{3}$\\
  \hline
\end{tabular} 
}
\end{table}
\begin{figure}[ht!]  
      \begin{tabular}{ccc} 
        \includegraphics[width=0.3\textwidth]{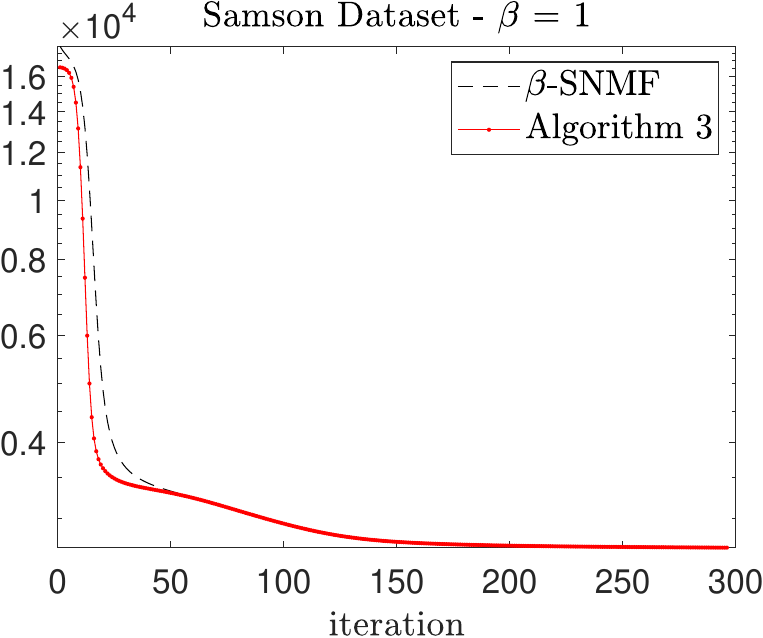}
    & 
        \includegraphics[width=0.3\textwidth]{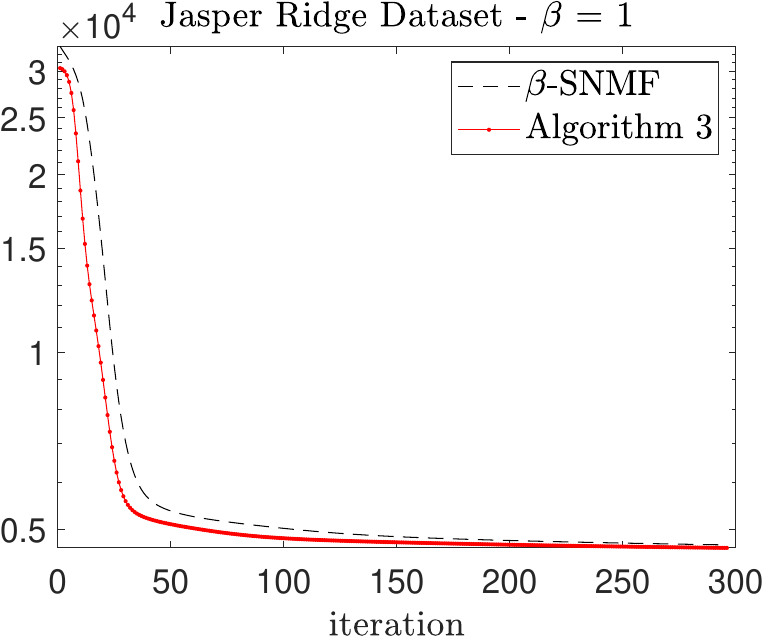}
    & 
        \includegraphics[width=0.312\textwidth]{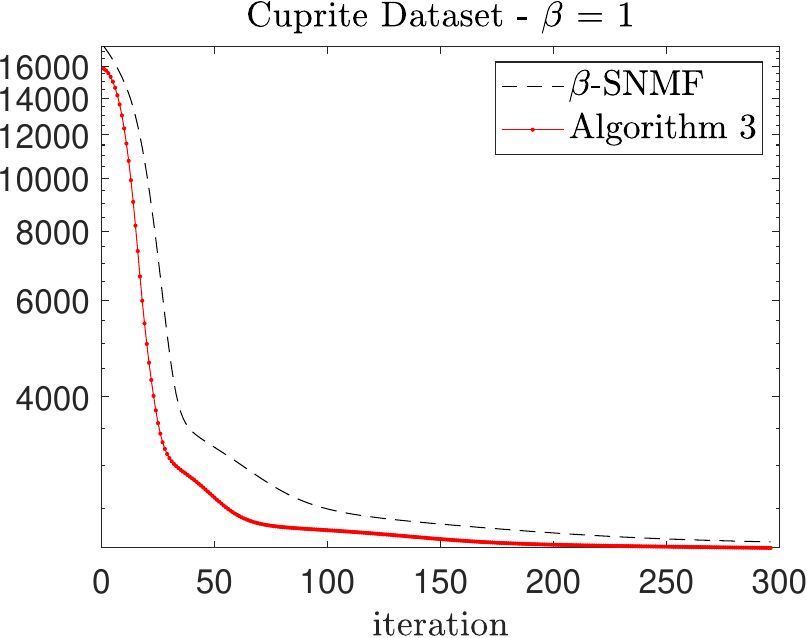}
    \\ 
        \includegraphics[width=0.3\textwidth]{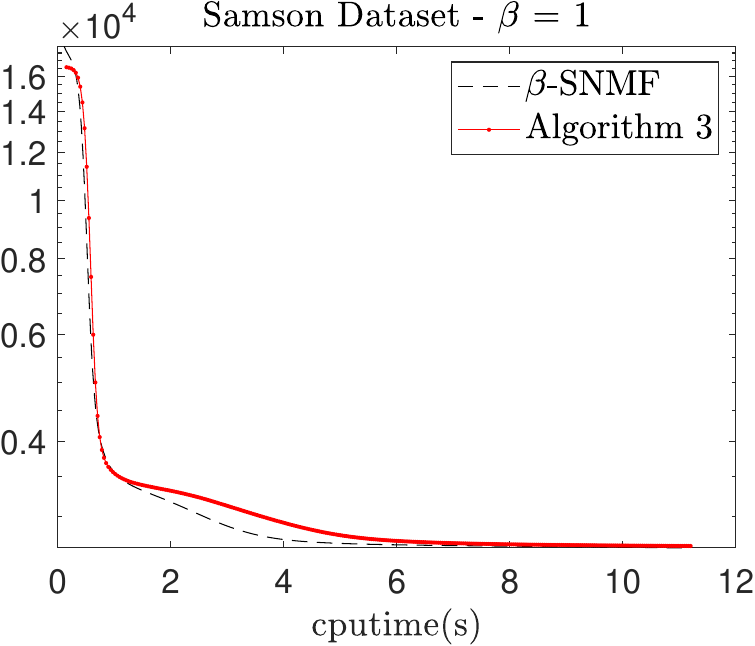}
    & 
        \includegraphics[width=0.3\textwidth]{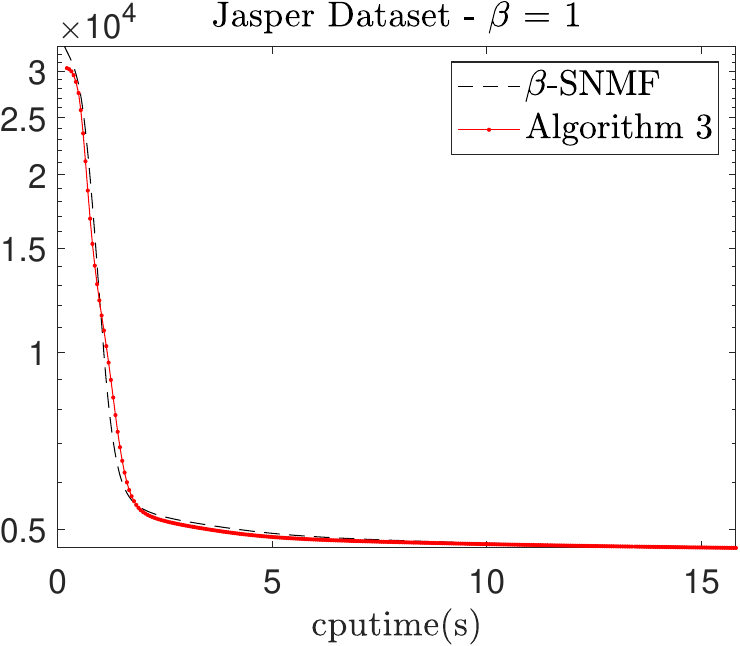}
   & 
        \includegraphics[width=0.312\textwidth]{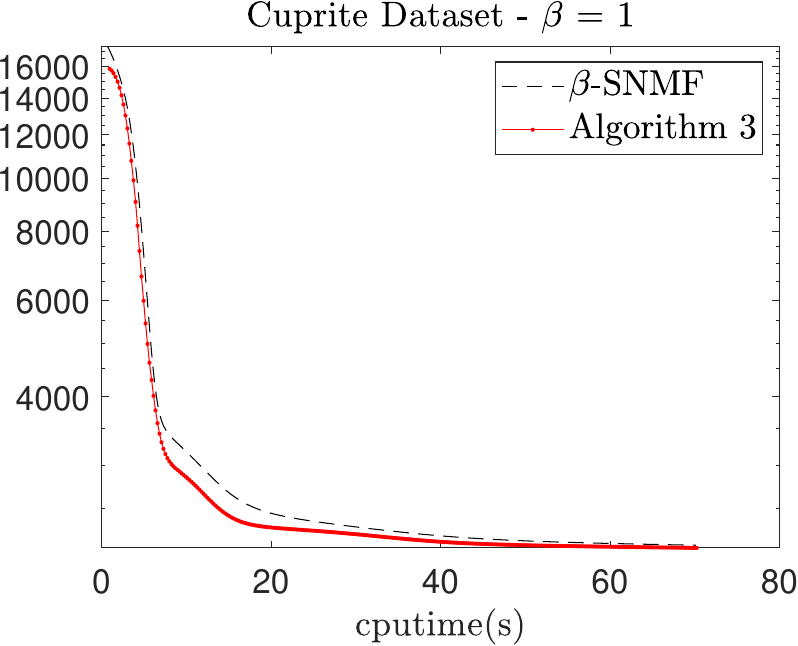}
    \end{tabular} 
    \caption{Averaged objective functions over 20 random initializations obtained for Algorithm~\ref{dchsminvolKLNMF} with 300 iterations (red line with circle markers), and the heuristic $\beta$-SNMF from \cite{Leroux} (black dashed line).  
    }\label{fig:hsbetaNMF_lossfun_compa}
\end{figure} 

According to Table~\ref{table:hsbetaNMFperf} (top row), we observe that  Algorithm~\ref{dchsminvolKLNMF} outperforms the heuristic from \cite{Leroux} in terms of final value for the objective functions while $\beta$-SNMF shows lower runtimes. 
Additionally, based on Figure~\ref{fig:hsbetaNMF_lossfun_compa}, we observe that Algorithm~\ref{dchsminvolKLNMF} converges on average faster than $\beta$-SNMF for all the data sets, in terms of iterations.  
However, $\beta$-SNMF has a lower computational cost per iteration.
Thus, we complete the comparison between both algorithms by imposing the same computational time: we run Algorithm~\ref{dchsminvolKLNMF} for 300 iterations, record the computational time and run $\beta$-SNMF for the same amount of time. 
\begin{table}[ht!]
\centering
\caption{Final value of objective function values $\Phi_\text{end}(W,H)$ for Algorithm~\ref{dchsminvolKLNMF} and the heuristic from \cite{Leroux}. The table reports the average and standard deviation over 20 random initializations for an equal computational time that corresponds to 300 iterations of Algorithm~\ref{dchsminvolKLNMF}.  A bold entry indicates the best value for each experiment. 
} 
\label{table:hsbetaNMFperf2}
\setlength{\tabcolsep}{2pt}
\begin{tabular}{|c|c|c|c|}
  \hline
  Algorithms       & Samson data set & Jasper Ridge data set & Cuprite data set\\
                   & $\Phi_\text{end}(W,H)$ & $\Phi_\text{end}(W,H)$ & $\Phi_\text{end}(W,H)$ \\
  \hline 
   Algorithm~\ref{dchsminvolKLNMF}   &(2.68$\pm$0.00)$10^{3}$  &\textbf{(4.65 $\pm$ 0.00)$10^{3}$} & \textbf{(2.12 $\pm$ 0.00)$10^{3}$} \\
  $\beta$-SNMF~\cite{Leroux}  &(2.68$\pm$0.00)$10^{3}$  &(4.66 $\pm$ 0.00)$10^{3}$  & (2.15 $\pm$ 0.00)$10^{3}$\\
  \hline
\end{tabular} 
\end{table}
Table~\ref{table:hsbetaNMFperf2} reports the average and standard deviation of the final value for the objective function over 20 runs in this setting.  Figure~\ref{fig:hsbetaNMF_lossfun_compa} (bottom row) displays the objective function w.r.t.\ time for the three data sets. 
On this comparison, Algorithm~\ref{dchsminvolKLNMF} and the heuristic from \cite{Leroux} perform similarly although Algorithm~\ref{dchsminvolKLNMF} has slightly better final objective function values. However, keep in mind that only Algorithm~\ref{dchsminvolKLNMF} is theoretically guaranteed to decrease the objective function. 
\section{Conclusion}\label{sec_conclusions}

In this paper we have presented a general framework to solve penalized $\beta$-NMF problems that integrates a set of disjoint constraints on the variables; see the general formulation~\eqref{eq:betaNMFmodel}.  
Using this framework, we showed that we can derive algorithms that compete favorably with the state of the art for a wide variety of $\beta$-NMF problems, such as  the simplex-structured NMF and the minimum-volume $\beta$-NMF with sum-to-one constraints on the columns of $W$. 
We have also shown how to extend the framework to non-linear disjoints constraints, with application to a sparse $\beta$-NMF model for $\beta=1$ where each column of $W$ lie on a hyper-sphere. 

Further works will focus on the possible extension of the methods to non-disjoints constraints. 
The non-disjoint constraints will lead to roots finding problems of  polynomial equations in the Lagrangian multipliers for which we hope to find conditions that ensure the uniqueness of the solution. 

Another interesting direction of research would be to apply our framework to other NMF models. For example, in probabilistic latent semantic analysis/indexing (PLSA/PLSI), the model is the following: given a nonnegative matrix $V$ such that $\eb^T V \eb = 1$ (this can be assumed w.l.o.g.\ by dividing the input matrix by $\eb^T V \eb$), solve 
    \[
    \max_{W \geq 0, H \geq 0, s \geq 0} 
    \sum_{f,n} v_{fn} \log(W\diag(s)H)_{fn} 
    \text{ such that } W^T \eb = \eb, H\eb = \eb, s^T \eb = 1 . 
    \] 
    This model is equivalent to KL-NMF~\cite{ding2008equivalence}, with the additional constraint that $\eb^T WH \eb = \eb^T X \eb$, and hence our framework is applicable to PLSA/PLSI. 
    Such constraints have also applications in soft clustering contexts; see~\cite{yang2016low}.


\section*{Acknowledgment}

We would like to thank the Associate Editor and the reviewers for taking the time to carefully read the paper and for the useful feedback that helped us improve the paper. 
We also thank Arthur Marmin for identifying an error in our derivations when $\beta \in (1,2)$ (indicated in red color in this version of the manuscript).



\appendix 

\section{Convexity, concavity and complete monotonicity for a convex-concave decomposition of the discrete $\beta$-divergence}\label{convconcacm}

The discrete $\beta$-divergence can always be expressed as the sum of convex, concave, and constant terms. In Table \ref{tab:decompbetadiv} we introduce a convex-concave decomposition of the $\beta$-divergence which slightly differ from the one given in \cite[Table\,1]{Fevotte_betadiv} (by the fact that ours contains no constant term $\widebar{d}$) as given in Table~\ref{tab:decompbetadiv}.

\begin{table}[h]
\centering
\ji{%
\setlength{\tabcolsep}{2pt}
\begin{tabular}{|c|c|c|c|c|c|}
\hline
\begin{tabular}{@{}c@{}}Decomposition\\$d_{\beta}=\widecheck{d}+\widehat{d}$\end{tabular}
&$\beta\in(-\infty,1)\setminus\{0\}$&$\beta=0$&$\beta=1$&$\beta\in(1,2)$ &   $\beta\in[2,+\infty)$    \\\hline
\phantom{\raisebox{8pt}{x}}
$\widecheck{d}(x|y)$& $\frac{1}{1-\beta} x \,y^{\beta -1}$    & $\frac{x}{y}$  & $-x\log y$  &  $\frac1{\beta}y^{\beta}-\frac1{\beta-1} xy^{\beta-1}$  & $\frac{1}{\beta} y^{\beta}$  \\\hline
\phantom{\raisebox{8pt}{x}}
$\widehat{d}(x|y)$&$ \frac{1}{\beta}y^{\beta}-\frac{1}{\beta\,(1-\beta)} x^{\beta}$&$\log\frac{y}{x} - 1$&$y + x\, \log x - x  $&$\frac{1}{\beta\,(\beta-1)} x^{\beta}$&$-\frac{1}{\beta-1} x \,y^{\beta -1}+\frac{1}{\beta\,(\beta-1)} x^{\beta}$\\\hline
\end{tabular}
}
\caption{
Proposed concave-convex decomposition of the discrete $\beta$-divergence.}
\label{tab:decompbetadiv}
\end{table}

In Table~\ref{tab:decompbetadiv}, $y\in(0,\infty)$, $\beta$ is real valued and $x\in(0,\infty)$. 
Further, $\beta$ and $x$ are considered as parameters, $d_{\beta}$, $\widehat{d}$ and $\widecheck{d}$ being handled as univariate functions of $y$.\\

Let us now recall the definition of a complete monotonic function $f$:
\begin{definition}
A function $f$ is said to be completely monotonic (c.m.) on an interval $I$ if $f$ has derivatives of all orders on $I$ and $(-1)^n f^{(n)}(x)\geq0$ for $x\in I$ and $n\geq0$.
\end{definition}
We can now introduce the properties of concavity, convexity and monotonicity for  our convex-concave formulation of the discrete $\beta$-divergence:

\begin{proposition} Given $\widecheck{d}(\cdot|\cdot)$ and $\widehat{d}(\cdot|\cdot)$ as defined above, we have that
\label{prop:approxdbeta}
\hfill
\begin{enumerate}
\item $\widecheck{d}(x|y)$ is $C^\infty$ and strictly convex on $(0,\infty)$ for $x>0$ and $\beta\in\mathbb{R}$;
\item $\widehat{d}(x|y)$ is concave for $x>0$ and $\beta\in\mathbb{R}$;
\item for all $\beta<2$, $\widecheck{d}''(x|y)$ and $\widehat{d}''(x|y)$ are c.m. 
\end{enumerate}
\end{proposition}
\begin{proof}
The proof is straightforward, given that $\widecheck{d}(x|y)$ and $\widehat{d}(x|y)$ linearly combine $C^\infty$ functions on $(0,\infty)$, and that in the same interval, 
\begin{itemize}
\item $\log y$ is strictly concave;
\item $y^\nu$ is strictly convex for all  $\nu\in(-\infty,0)\cup(1,\infty)$, and strictly concave for all $\nu\in(0,1)$;
\item $y^\nu$ is c.m. for all $\nu<0$.
\end{itemize}
\end{proof}
According to the first two items of Proposition~\ref{prop:approxdbeta}, $\widecheck{d}$ and $\widehat{d}$ indeed yield a convex-concave decomposition of the $\beta$-divergence, which is a variant of \cite[Table\,1]{Fevotte_betadiv}. Let us remark that the successive minimization of an upper approximation of this convex-concave decomposition following the methodology presented in \cite{Fevotte_betadiv} yields to the usual multiplicative update scheme.

\small 
\bibliographystyle{siam}
\bibliography{Bibliography}


\normalsize 

\newpage 

\section*{Supplementary Material} 

This supplementary materials provide additional numerical experiments. 
In S1, we show the evolution of the error as a function of the iterations for Algorithm~\ref{consDisBetaNMF} and GR-NMF~\cite{7194802} for the tests performed in Section~\ref{sec_ssbetaNMF}. 
In S2, we provide qualitative results obtained with Algorithm~\ref{dchsminvolKLNMF} on hyperspectral images. 

\subsection*{S1. Evolution of the objective function for $\beta$-SSNMF} \label{appen_quali_res_betaSSNMF}

Figure~\ref{fig:ssbetaNMF_lossfun_compa_part1} displays the  evolution of the {relative} objective function values, that is, $\frac{D_\beta(V|WH)}{D_\beta(V|v \eb\eb^T)}$, of $\beta$-SSNMF for Algorithm~\ref{consDisBetaNMF} and GR-NMF~\cite{7194802} on the experiments described in Section~\ref{sec_ssbetaNMF}.  
As mentioned in the paper, Algorithm~\ref{consDisBetaNMF} performs better than   GR-NMF~\cite{7194802}, except for $\beta = 0$.  

\begin{figure}
    \centering  
    \begin{subfigure}[b]{0.28\textwidth}
        \includegraphics[width=\textwidth]{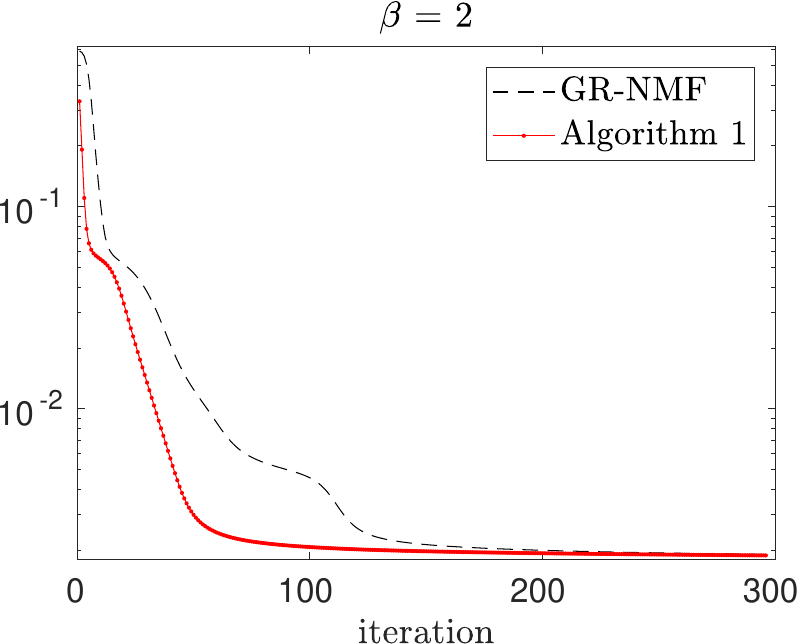}
        \includegraphics[width=\textwidth]{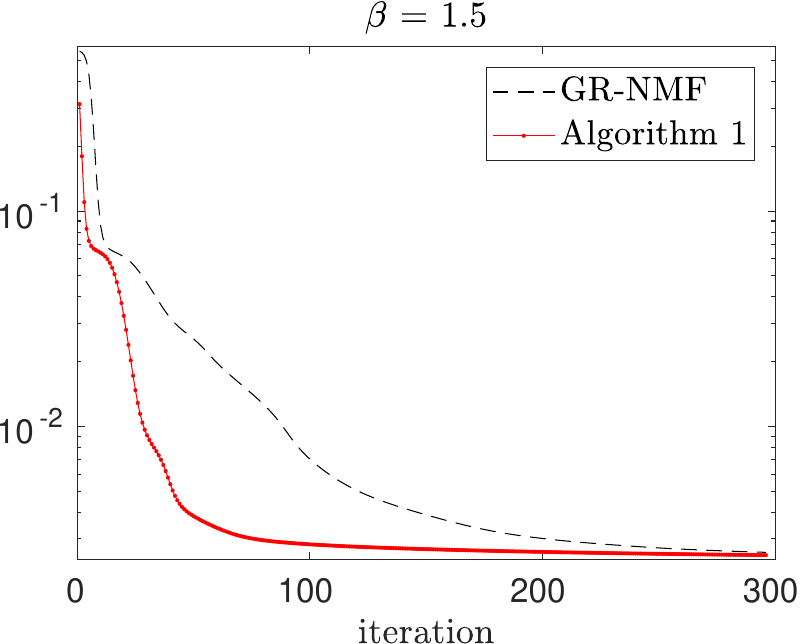}
        \includegraphics[width=\textwidth]{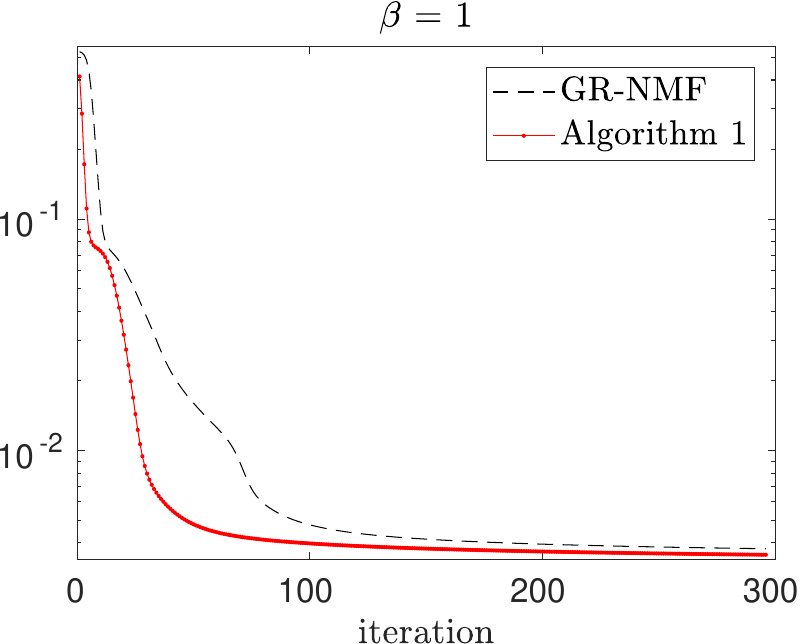}
        \includegraphics[width=\textwidth]{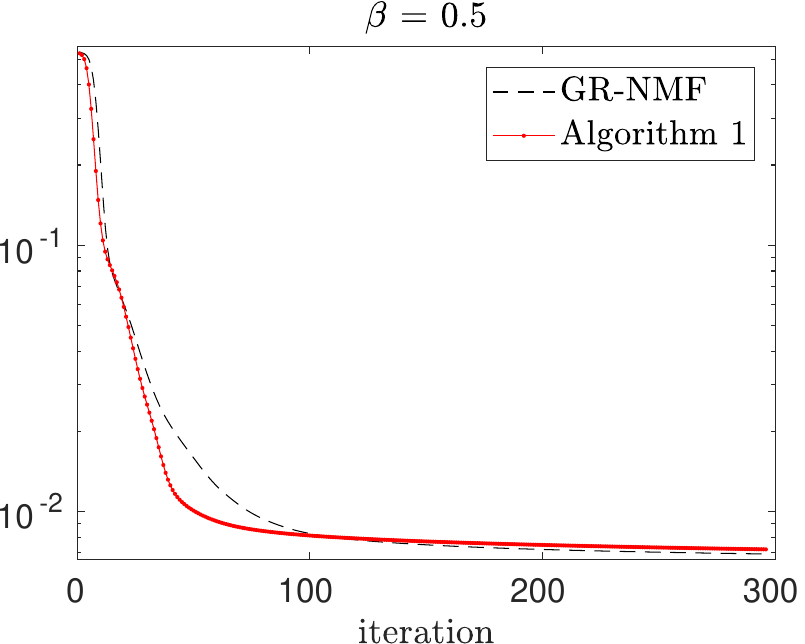}
        \includegraphics[width=\textwidth]{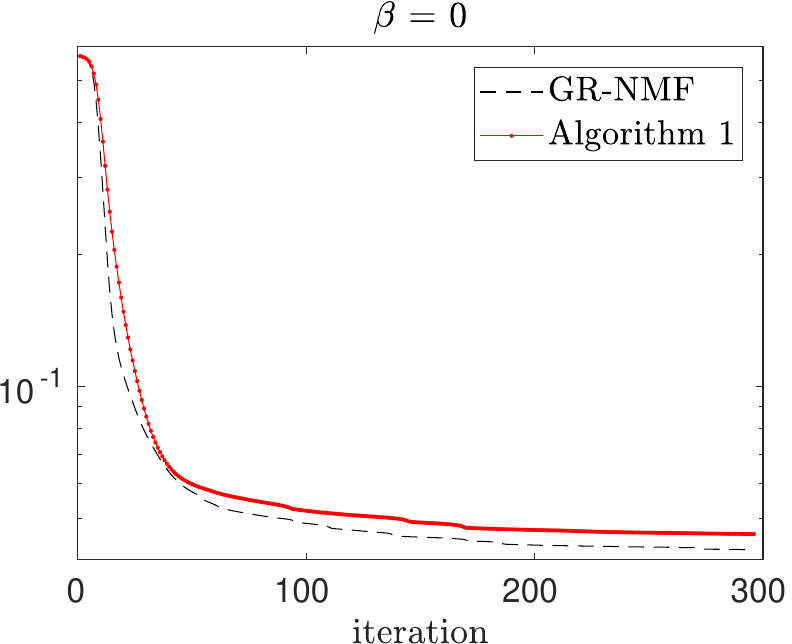}
        \caption{Samson Data set}
    \end{subfigure}
~ 
    \begin{subfigure}[b]{0.2815\textwidth}
        \includegraphics[width=\textwidth]{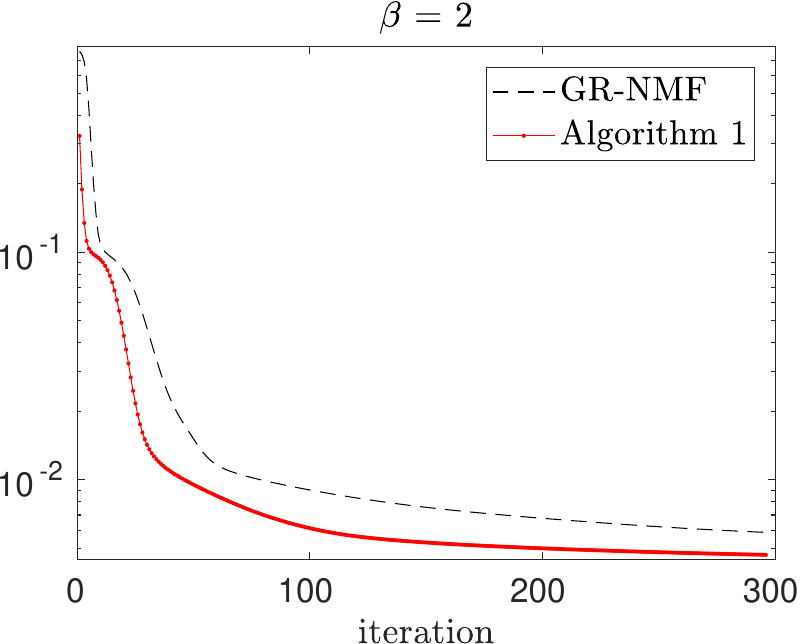}
        \includegraphics[width=\textwidth]{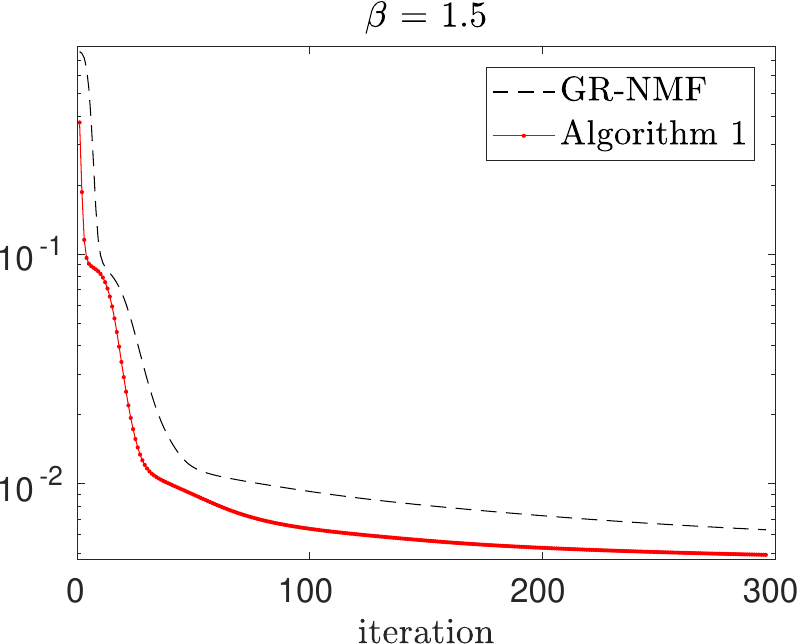}
        \includegraphics[width=\textwidth]{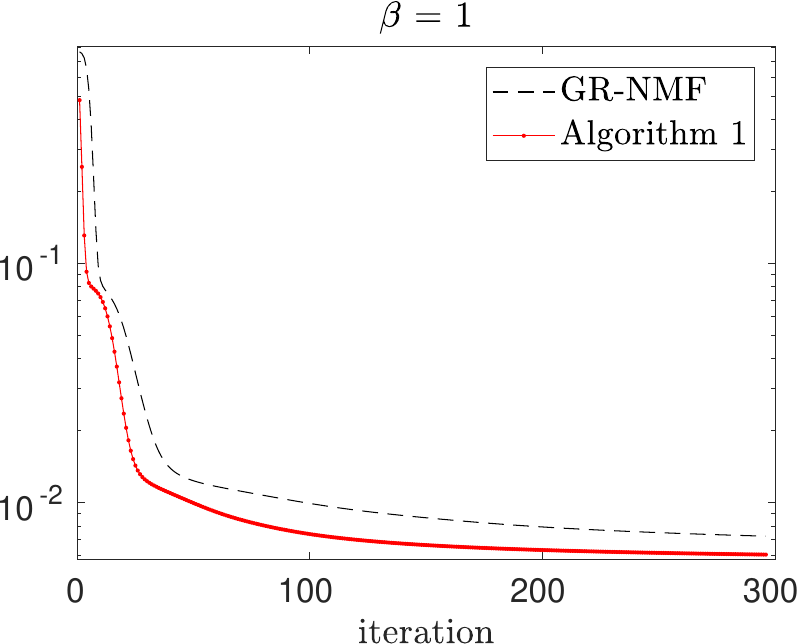}
        \includegraphics[width=\textwidth]{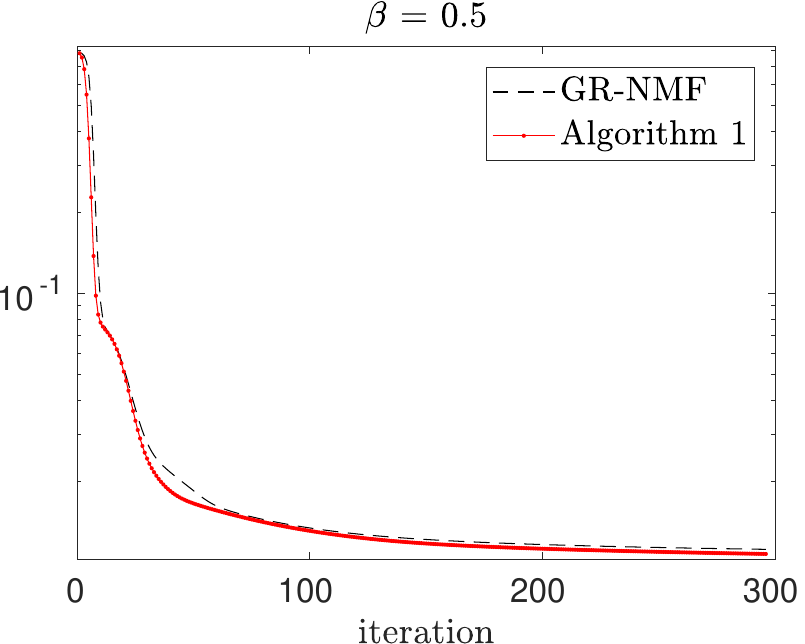}
        \includegraphics[width=\textwidth]{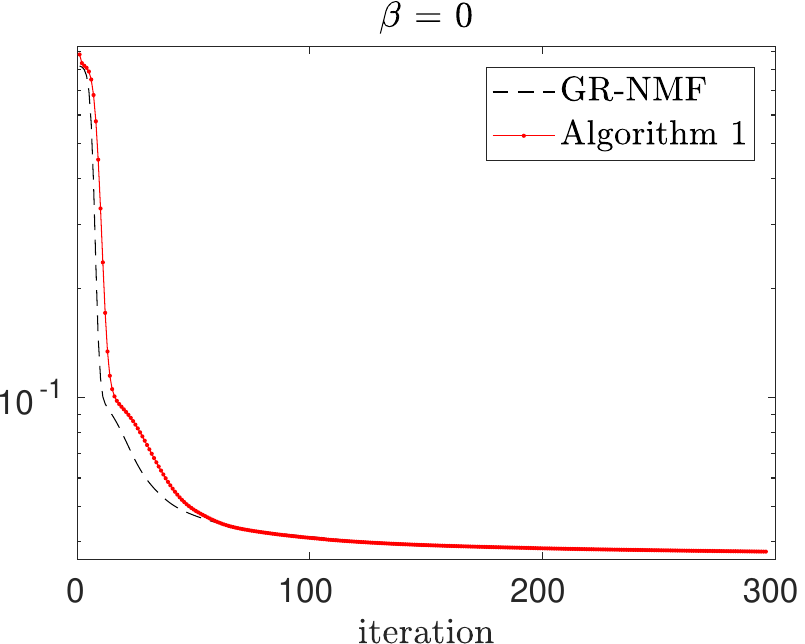}
        \caption{Jasper Ridge Data set}
    \end{subfigure}
    \begin{subfigure}[b]{0.2815\textwidth}
        \includegraphics[width=\textwidth]{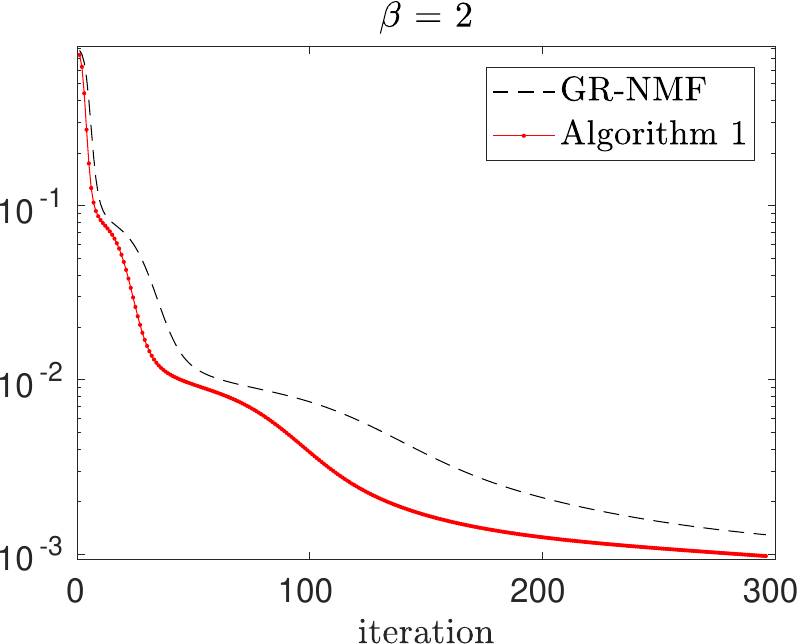}
        \includegraphics[width=\textwidth]{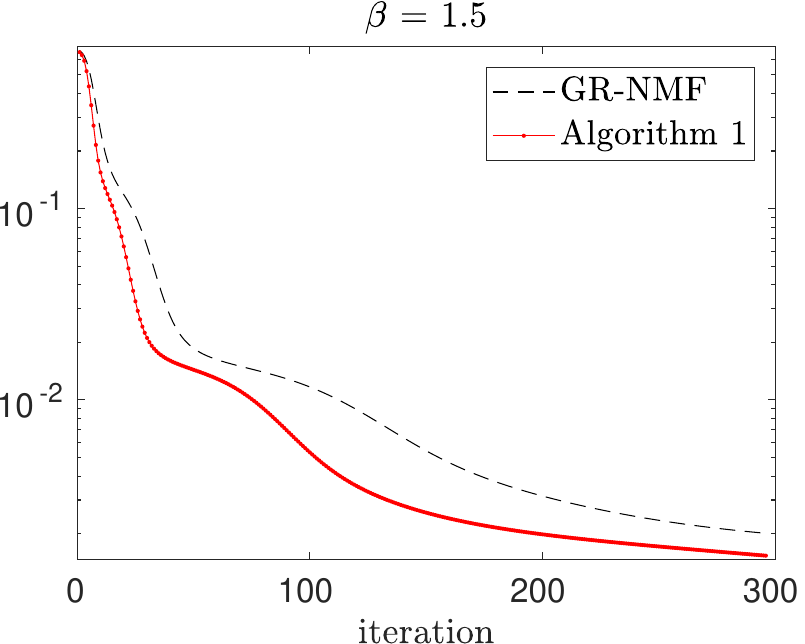}
        \includegraphics[width=\textwidth]{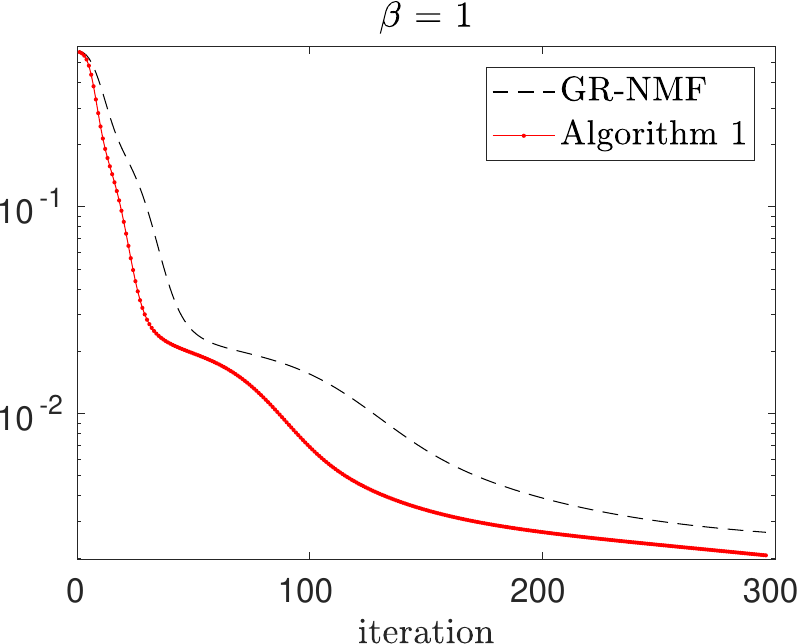}
        \includegraphics[width=\textwidth]{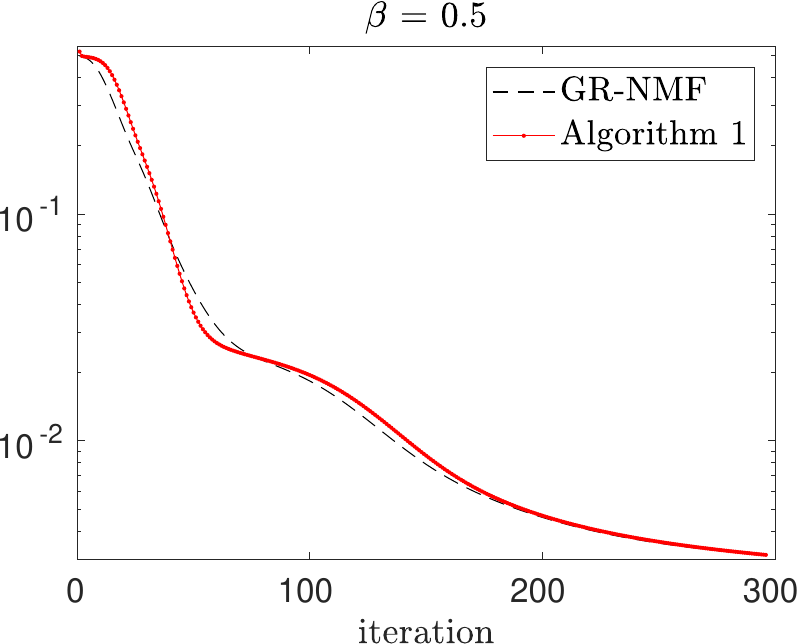}
        \includegraphics[width=\textwidth]{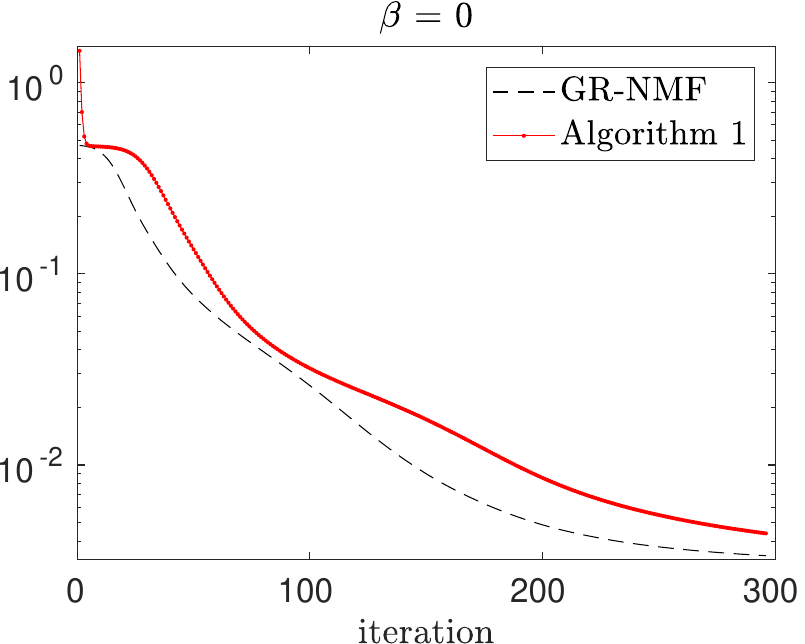}
        \caption{Cuprite Data set}
    \end{subfigure}
    \caption{Averaged {relative} objective function values over 20 random initializations obtained for Algorithm~\ref{consDisBetaNMF} (red line with circle markers) and the GR-NMF (black dashed line) applied to the three data sets detailed in the text for 300 iteration. The comparison is performed for different values of $\beta$, from top to bottom: $\beta=2$, $\beta=3/2$, and $\beta=1$. Logarithmic scale for y axis. 
}\label{fig:ssbetaNMF_lossfun_compa_part1}
\end{figure}

\subsection*{S2. Qualitative results obtained with Algorithm~\ref{dchsminvolKLNMF}} \label{appen_quali_res_algoKLsparse}

In the following we report qualitative results obtained with Algorithm~\ref{dchsminvolKLNMF} applied to three HS real data sets, that are Samson, Jasper and Urban data sets. The first two data sets are detailed in Section \ref{sec_ssbetaNMF}. The Urban data set contains 162 spectral bands with 307$\times$307 pixels with mostly six endmembers. 
Note that Cuprite data set is replaced by the Urban data set since endmembers for Cuprite correspond to chemical components which are more difficult to interpret visually while endmembers for Urban data sets are more easily interpretable.

As mentioned earlier, $\lambda_{k}$ enables to control sparsity of the $k$-th row of $H$. Given a row $H(k,:) \in \mathbb{R}_{+}^{N}$ of $H$, a meaningful way to measure its sparsity is to consider the following measure \cite{1055559}:
\begin{equation}\label{eq:sparlevel}
\begin{aligned}
& \text{sp}\left( H(k,:) \right)=\frac{\sqrt{N}-\frac{\left\| H(k,:) \right\|_{1}}{\left\| H(k,:) \right\|_{2}}}{\sqrt{N}-1} \in \left[0,1 \right].
\end{aligned}
\end{equation}

During the numerical experiments, we observed that Algorithm~\ref{dchsminvolKLNMF} gives better results when the initial values for $\lambdab$ are low and progressively increased.  
During a specified interval of iterations $\left[it_{\min}, it_{\max}\right]$, the sparsity of the current iterate is measured by using equation \eqref{eq:sparlevel}, and the entries of $\lambdab$ are dynamically updated (increased with a rate $\alpha >1$) to achieve a desired sparsity level $sp$. 
The dynamic update of the weight vector to reach the desired levels of sparsity has been activated in the iterations intervals $\left[1,150\right]$, $\left[1,150\right]$ and $\left[1,75\right]$ for Samson, Jasper and Urban,  respectively. We report here the abundance maps of each end-member for two levels of average target sparsity that are 0.25 and 0.5. For all the simulations, the weight vector $\lambdab$ has been initialized to $0.05\eb$, and the algorithm was run for 300~iterations. 

We fix the number of endmembers to 3, 4 and 6 respectively for Samson, Jasper Ridge and Urban data sets, these values are commonly considered in the HS community \cite{zhu2017hyperspectral}. 
Figures \ref{fig:abundances_Samson} to \ref{fig:abundances_Urban} picture the abundance estimation for the three data sets for the two levels of sparsity. 
\begin{figure}
    \centering  
    \begin{subfigure}[b]{\textwidth}
        \centering 
        \includegraphics[width=0.3\textwidth]{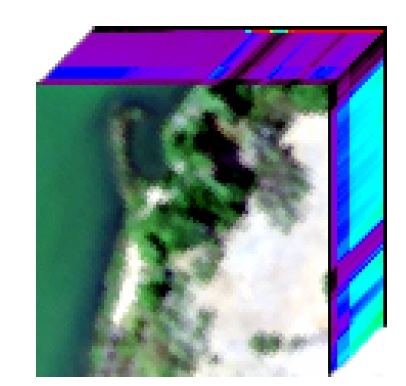}
        \caption{Samson Data set}
    \end{subfigure}\\
    \begin{subfigure}[b]{\textwidth}
        \centering 
        \fbox{\includegraphics[width=0.8\textwidth]{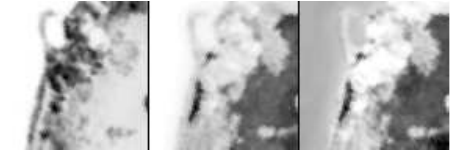}}
        \caption{
        Abundance map with average sparsity level set to 0.25}
    \end{subfigure}\\
    \begin{subfigure}[b]{\textwidth}
        \centering 
        \fbox{\includegraphics[width=0.8\textwidth]{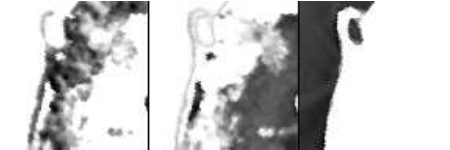}}
        \caption{Abundance map with average sparsity level set to 0.5}
    \end{subfigure}
    \caption{Samson data set (a) and results ((b) and (c)) for the Abundance maps estimated using Algorithm~\ref{dchsminvolKLNMF} for the three endmembers: $\sharp$1 Tree, $\sharp$2 Soil and $\sharp$3 Water. Two average sparsity levels considered: 0.25 (b) and 0.5 (c). 
    }\label{fig:abundances_Samson}
\end{figure}

\begin{figure}
    \centering  
    \begin{subfigure}[b]{\textwidth}
        \centering 
        \includegraphics[width=0.30\textwidth]{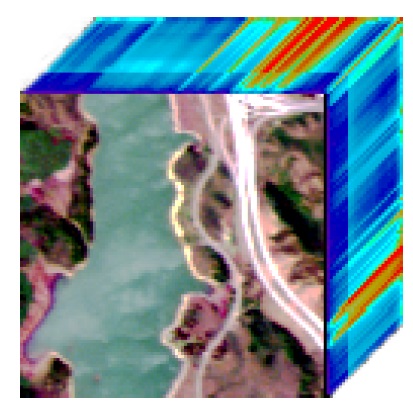}
        \caption{Jasper Ridge Data set}
    \end{subfigure}\\
    \begin{subfigure}[b]{\textwidth}
        \centering
        \includegraphics[width=\textwidth]{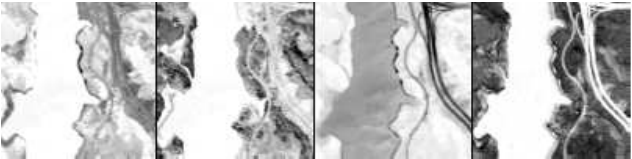}
        \caption{Abundance map with average sparsity level set to 0.25}
    \end{subfigure}\\
    \begin{subfigure}[b]{\textwidth}
        \centering
        \includegraphics[width=\textwidth]{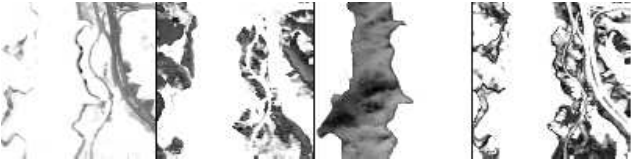}
        \caption{Abundance map with average sparsity level set to 0.5}
    \end{subfigure}
    \caption{Jasper Ridge data set (a) and results ((b) and (c)) for the Abundance maps estimated using Algorithm~\ref{dchsminvolKLNMF} for the four endmembers: $\sharp$1 Road, $\sharp$2 Tree,  $\sharp$3 Water and $\sharp$4 Soil. Two average sparsity levels are considered: 0.25 (b) and 0.5 (c).}\label{fig:abundances_Jasper}
\end{figure}

\begin{figure}
    \centering  
    \begin{subfigure}[b]{0.22\textwidth}
        \centering 
        \includegraphics[width=\textwidth]{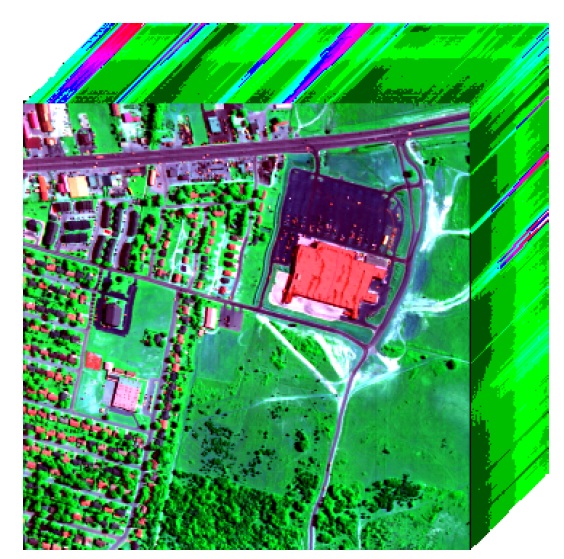}
        \caption{Urban Data set}
    \end{subfigure}\\
    \begin{subfigure}[b]{0.6\textwidth}
        \centering
        \includegraphics[width=\textwidth]{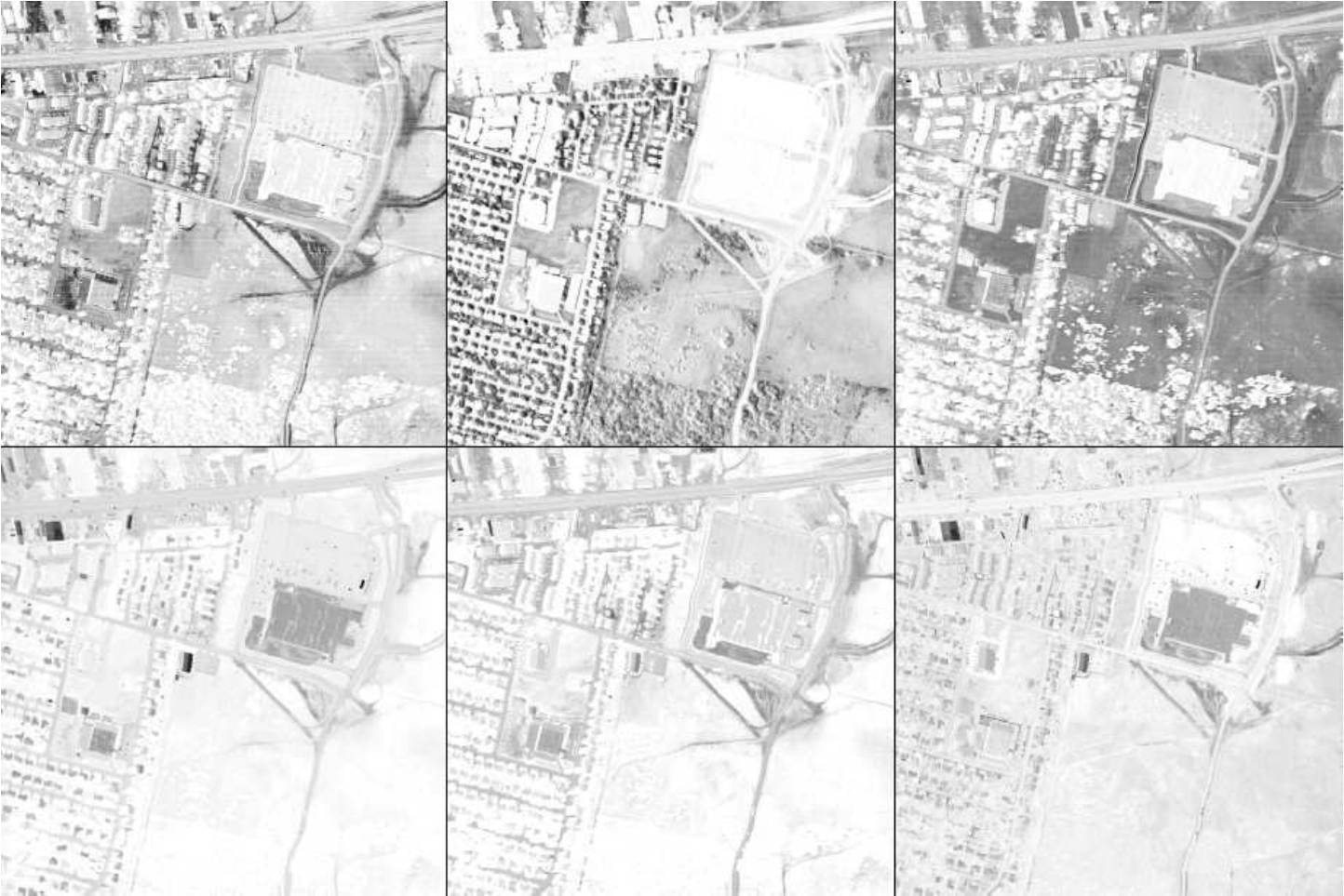}
        \caption{Abundance map with average sparsity level set to 0.25}
    \end{subfigure}\\
    \begin{subfigure}[b]{0.6\textwidth}
        \centering
        \includegraphics[width=\textwidth]{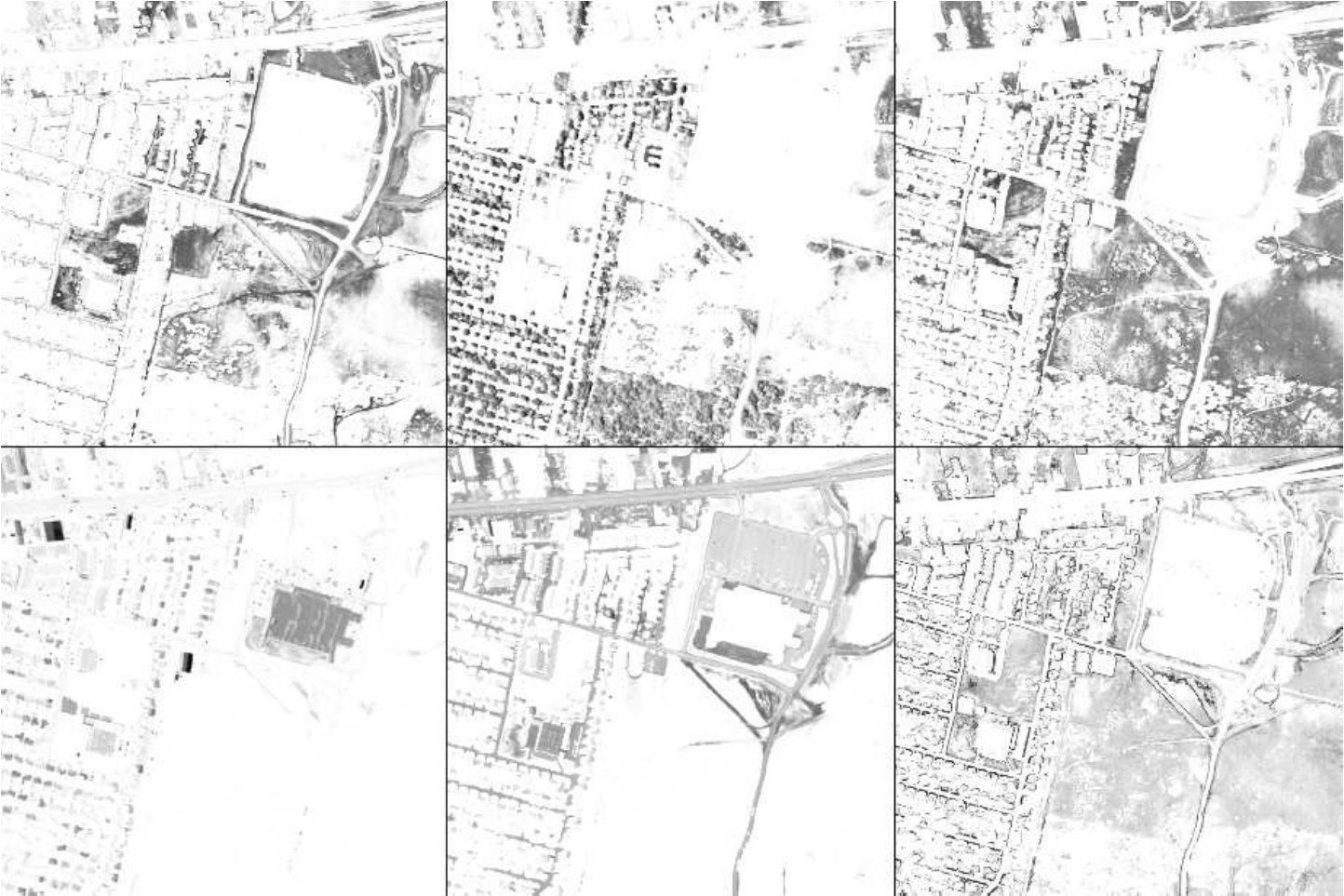}
        \caption{Abundance map with average sparsity level set to 0.5}
    \end{subfigure}
    \caption{Urban data set (a) and results ((b) and (c)) for the Abundance maps estimated using Algorithm~\ref{dchsminvolKLNMF} for the six endmembers: $\sharp$1 Soil, $\sharp$2 Tree,  $\sharp$3 Grass, $\sharp$4 Roof, $\sharp$5 Road/Asphalt and $\sharp$6 Roof2/shadows. Two average sparsity levels are considered: 0.25 (b) and 0.5 (c).
    }\label{fig:abundances_Urban}
\end{figure} 

\begin{figure}[h!]
      \centering
	  \includegraphics[width=0.8\linewidth]{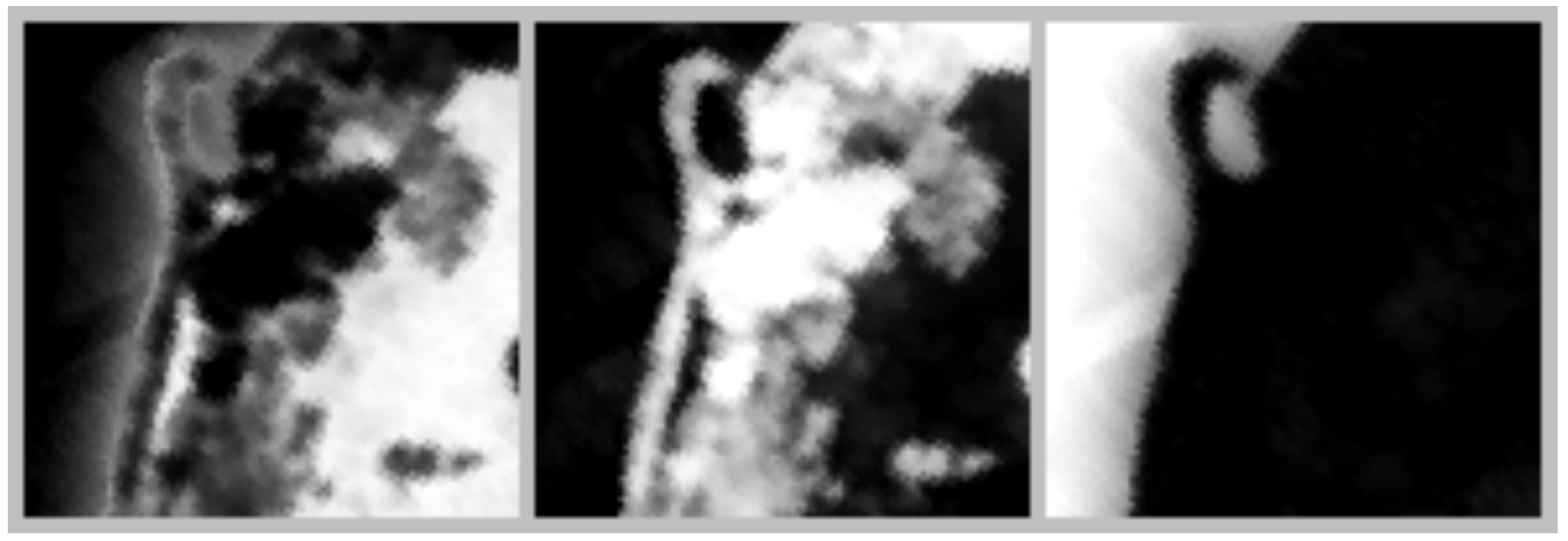}
	  \caption{Baseline abundances for the endmembers obtained for Samson data extracted from \cite{zhu2017hyperspectral}: $\sharp$1 Soil, $\sharp$2 Tree and $\sharp$3 Water. }\label{fig:abundances_Samson_GT}
\end{figure}

\begin{figure}[h!]
      \centering
	  \includegraphics[width=0.9\linewidth]{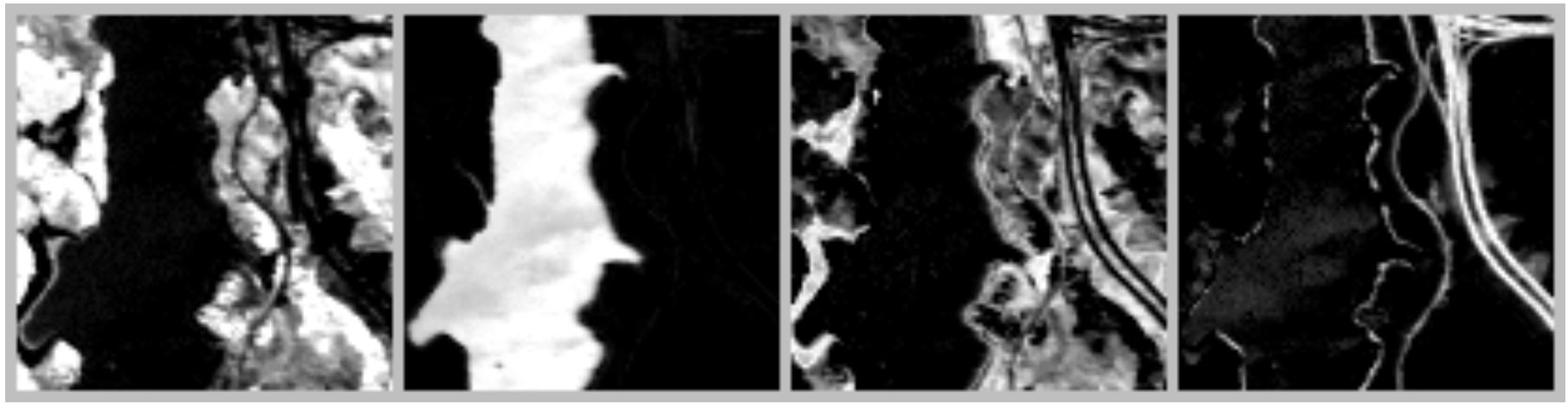}
	  \caption{Baseline abundances for the endmembers obtained for Jasper Ridge data extracted from \cite{zhu2017hyperspectral}: $\sharp$1 Road, $\sharp$2 Soil, $\sharp$3 Water and $\sharp$4 Tree. }\label{fig:abundances_Jasper_GT}
\end{figure}

\begin{figure}[h!]
      \centering
	  \includegraphics[width=0.6\linewidth]{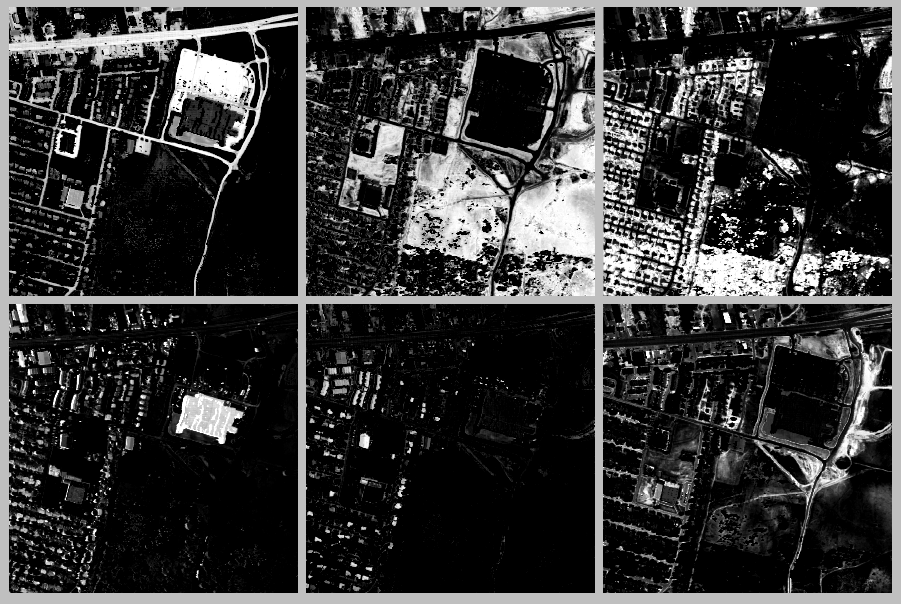}
	  \caption{Baseline abundances for the endmembers obtained for Urban data extracted from \cite{zhu2017hyperspectral}: $\sharp$1 Asphalt, $\sharp$2 Grass, $\sharp$3 Tree, $\sharp$4 Roof1, $\sharp$5 Roof2/Shadow and $\sharp$6 Soil.}\label{fig:abundances_Urban_GT}
\end{figure}

In order to validate the results obtained for the abundances of the endmembers, we display in Figures \ref{fig:abundances_Samson_GT}, \ref{fig:abundances_Jasper_GT} and \ref{fig:abundances_Urban_GT} the ground truth results obtained in \cite{zhu2017hyperspectral}. Note that the grayscale used in \cite{zhu2017hyperspectral} is the complementary of the one used in Figures \ref{fig:abundances_Samson} to \ref{fig:abundances_Urban}.

We observe that the abundance estimation gets significantly more accurate when the level of average sparsity is higher. For the Samson and Jasper Ridge data sets, the abundances for the endmembers are nicely estimated while five endmembers over six are well estimated for the Urban data set. 
The ``Roof" is divided into ``Roof1"  and  ``Roof2/shadow" \cite{5871318,zhu2017hyperspectral}. In our simulations, it seems that the sixth endmember corresponds to some shadows with a small residual of ``Grass", while the ``Roof" is not split into two groups.

\end{document}